\documentclass{article}

\usepackage[main, final]{neurips_2026}

\usepackage[utf8]{inputenc} 
\usepackage[T1]{fontenc}    
\usepackage{hyperref}       
\usepackage{url}            
\usepackage{booktabs}       
\usepackage{amsfonts}       
\usepackage{nicefrac}       
\usepackage{microtype}      
\usepackage{xcolor}         
\usepackage{amsthm}
\usepackage{amssymb}
\usepackage{booktabs}
\usepackage{multirow}
\usepackage{float}
\usepackage{makecell}
\newtheorem{theorem}{Theorem}

\newtheorem{lemma}{Lemma}

\usepackage{bm}
\usepackage{subfigure}
\usepackage{colortbl}
\usepackage{amsmath}
\usepackage{listings}
\usepackage{caption}
\usepackage{tcolorbox}
\tcbuselibrary{skins}

\makeatletter
\@ifundefined{c@listing}{
  \DeclareCaptionType{listing}[Listing][List of Listings]
}{}
\makeatother

\lstdefinestyle{minprocode}{
    language=Python,
    basicstyle=\ttfamily\footnotesize,
    keywordstyle=\color{blue!65!black}\bfseries,
    commentstyle=\color{black!48}\itshape,
    stringstyle=\color{teal!70!black},
    backgroundcolor=\color{black!1},
    numberstyle=\ttfamily\tiny\color{black!40},
    numbers=left,
    numbersep=8pt,
    firstnumber=1,
    columns=fullflexible,
    keepspaces=true,
    tabsize=4,
    breaklines=true,
    breakindent=1.5em,
    showstringspaces=false,
    frame=none,
    float=false,
    xleftmargin=2.2em,
    aboveskip=0pt,
    belowskip=0pt
}
\usepackage{footmisc}  
\makeatletter
\renewcommand{\footnoterule}{%
  \kern -3pt 
  \hrule width \textwidth height 0.4pt
  \kern 2pt  
}
\makeatother


\usepackage{lineno}

\definecolor{darkblue}{rgb}{0, 0, 0.5}
\definecolor{darkgreen}{RGB}{50,100,0}
\definecolor{darkred}{RGB}{200, 0, 0}
\definecolor{lightblue}{RGB}{220,235,250}
\usepackage[tikz]{bclogo}
\usepackage{enumitem}
\newenvironment{itemize*}%
 {\leftmargini=20pt\begin{itemize}%
  \setlength{\itemsep}{3pt}%
  \setlength{\parskip}{0pt}%
  }%
 {\end{itemize}}
\newenvironment{enumerate*}%
 {\begin{enumerate}%
  \setlength{\itemsep}{0pt}%
  \setlength{\parskip}{0pt}}%
 {\end{enumerate}}

\newcolumntype{P}[1]{>{\centering\arraybackslash}p{#1}}

\usepackage[nameinlink]{cleveref} 

\title{Stabilizing Off-policy LLM Optimization with \\ Prefix Importance Ratio}

\author{%
  Shiye Lei$\ ^{1}$  \qquad \qquad Zhihao Cheng$\ ^{1}$\thanks{Corresponding author: \href{mailto:zhihao.cheng@bytedance.com}{zhihao.cheng@bytedance.com}}  \qquad \qquad Dacheng Tao$\ ^{2}$ \\
  \\[0.05pt]
  $^{1}$ ByteDance BandAI \\
  $^{2}$ College of Computing \& Data Science, Nanyang Technological University 
}

\begin{document}

\maketitle

\begin{abstract}

Reinforcement learning (RL) post-training has increasingly demonstrated strong ability to elicit reasoning behaviors in large language models (LLMs). For training efficiency, rollouts are typically generated in an off-policy manner using an older sampling policy and then used to update the current target policy. To correct the resulting discrepancy between the sampling and target policies, most existing RL objectives rely on a token-level importance sampling ratio. However, we observe unstable training dynamics in existing importance-weighted objectives under large policy lag. In this paper, we revisit LLM policy optimization under off-policy conditions and show that the exact correction for the prefix-marginal policy gradient is the cumulative prefix importance ratio. Token-level weighting does not explicitly correct for this prefix distribution shift. To stabilize LLM optimization under large off-policy drift, we propose a simple yet effective objective, {\it Minimum Prefix Ratio} (\textbf{MinPRO}). MinPRO replaces the unstable cumulative prefix ratio with a biased, non-cumulative surrogate based on the minimum token-level ratio observed in the preceding prefix. Extensive experiments on both dense and mixture-of-experts LLMs, across multiple mathematical reasoning benchmarks, demonstrate that MinPRO substantially improves training stability and peak performance in off-policy regimes.
\end{abstract}

\begin{figure}[h]
\centering
\subfigure[An overview of MinPRO]{
\label{figure:sketch}
\includegraphics[width=0.51\columnwidth]{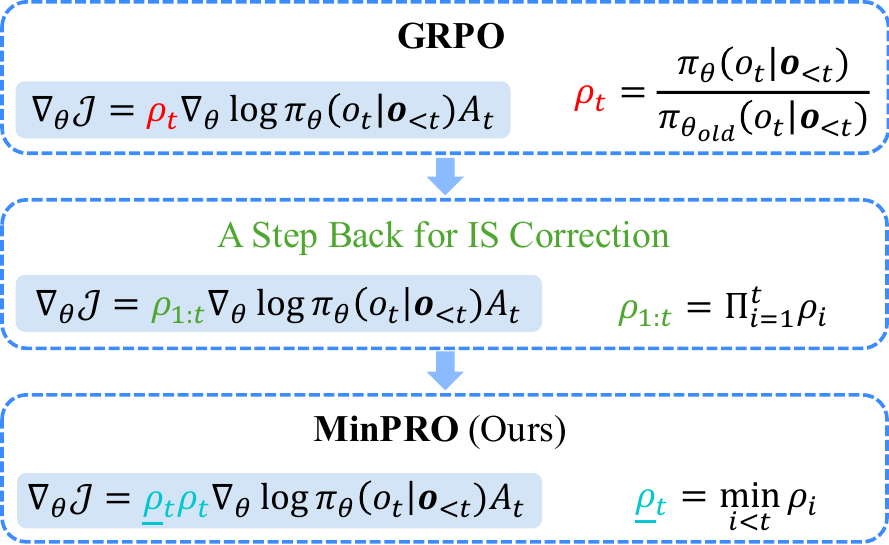}
}
\hfill
\subfigure[AIME24 scores under off-policy training]{
\label{figure:aime24-reward}
\includegraphics[width=0.44\columnwidth]{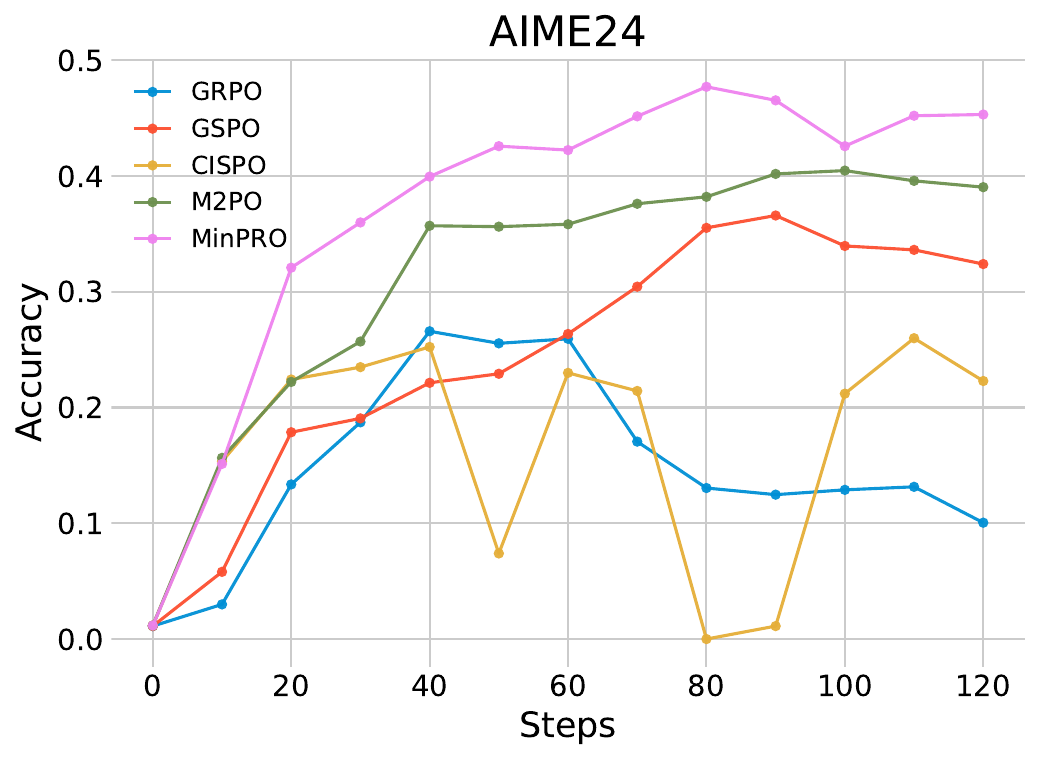}
}
\caption{(a) An overview of MinPRO. $\nabla_\theta \mathcal{J}$ denotes the policy gradient and $\rho_t$ is the token-level importance sampling (IS) ratio. We take a step back to derive the full IS ratio $\rho_{1:t}$, referred to as the prefix importance ratio. MinPRO is then developed by relaxing $\rho_{1:t} = \rho_{1:t-1} \cdot \rho_t$ to a non-cumulative proxy $\underline{\rho}_t \rho_t$. (b) AIME24 scores as functions of training steps for Qwen3-30B-A3B-Base under off-policy training.}
\label{figure:main}
\end{figure}

\section{Introduction}
\label{sec:intro}
Post-training has become an essential stage in modern large language model (LLM) development, complementing base model pre-training by aligning model behavior with human objectives and improving reasoning quality \citep{ouyang2022training,jaech2024openai}. In many practical scenarios, only final ground truth answers are available, and collecting detailed chain-of-thought annotations is expensive and often impractical. Reinforcement learning (RL) has therefore emerged as the most effective and widely used framework for LLM post-training \citep{guo2025deepseek}. 

When optimizing LLMs, an off-policy workflow is typically adopted: rollouts are generated by an older behavior policy $\pi_{\mathrm{old}}$, while the optimization target is a newer policy $\pi$. This design is primarily driven by system efficiency considerations. In practice, rollouts are often generated in large batches, which helps alleviate GPU idle time caused by response-length imbalance and improves infrastructure throughput \citep{team2025kimi,gao2025rollpacker}. These rollout batches are then divided into multiple mini-batches for gradient updates. Moreover, recent asynchronous RL frameworks further decouple rollout generation from policy gradient updates, introducing additional policy lag and amplifying off-policyness \citep{noukhovitch2024asynchronous,roux2025tapered}. Although operating in a large off-policy regime can substantially improve rollout efficiency, the resulting rollout distribution shift poses severe challenges to optimization stability and often leads to training collapse under high off-policyness \citep{xi2025bapo,zheng2025prosperity}.

To correct the rollout distribution discrepancy that arises in off-policy training, a standard approach is to apply the importance sampling ratio $\pi / \pi_{\mathrm{old}}$. Most LLM RL objectives instantiate this correction at the token level. However, by revisiting the policy gradient in the off-policy setting, we show that the theoretically justified correction is the prefix importance ratio, rather than the token-level ratio. Although the token-level relaxation can be effective in near-on-policy regimes, it does not accurately capture policy mismatch for LLMs, where rollouts are typically long. Consequently, this approximation can induce unstable optimization and lead to inferior performance, or even training collapse, under substantial off-policy drift.

Motivated by this insight, we seek to incorporate prefix-level information to stabilize policy optimization. In autoregressive LLMs, the prefix importance ratio is defined as a cumulative product, which is often numerically unstable and therefore impractical to use directly. To address this challenge, we introduce {\it Minimum Prefix Ratio} (\textbf{MinPRO}), a simple yet effective objective to stabilize LLM post-training. MinPRO replaces the cumulative prefix ratio with a simple surrogate: the current token ratio multiplied by the minimum token ratio observed in the preceding prefix (see \Cref{figure:sketch}). This formulation preserves essential prefix-level information while mitigating numerical instability and length bias by avoiding explicit cumulative products.

To comprehensively evaluate the stability of MinPRO, we conduct RL post-training under a large off-policy regime and compare it against several widely used and strong baselines, including GRPO \citep{shao2024deepseekmath,yu2025dapo}, GSPO \citep{zheng2025group}, CISPO \citep{chen2025minimax}, and M2PO \citep{zheng2025prosperity}. As pre-trained models, we consider two dense LLMs, Qwen3-8B-Base and Qwen3-14B-Base, as well as the mixture-of-experts (MoE) model Qwen3-30B-A3B-Base \citep{yang2025qwen3}. We assess performance on a suite of mathematical reasoning benchmarks, including AMC23, AIME24, AIME25, MATH500, Olympiad, Minerva, and GSM8K. From the training reward curves, which play a role analogous to training loss curves in supervised learning, MinPRO consistently exhibits higher rewards and more stable optimization dynamics throughout training. Evaluation on downstream benchmarks further shows that MinPRO achieves the highest peak performance among all compared methods, demonstrating strong stability under off-policy conditions.

\section{Related Work}

\paragraph{LLM Policy Optimization} Pre-training equips LLMs with broad factual and linguistic knowledge, yet additional techniques are required to strengthen their reasoning ability. Chain-of-thought (CoT) prompting, which encourages models to decompose problems into intermediate steps, has emerged as an effective approach for improving reasoning performance \citep{wei2022chain}. More recently, RL-based post-training has been shown to induce CoT behaviors without requiring explicit step-by-step annotations. Seminal systems such as OpenAI-o1 \citep{jaech2024openai} and DeepSeek-R1 \citep{guo2025deepseek} leverage RL to incentivize deliberate multi-step reasoning, leading to substantial gains on challenging tasks including mathematical problem solving and code generation. Many RL objectives have been developed from different perspectives, including clip-range design \citep{yu2025dapo,yang2025dcpo,sheng2025espo}, token ratio \citep{wang2025aspo}, and token entropy \citep{cui2025entropy,wang2025beyond,lei2025revisiting}. In contrast to hard clipping, which discards all tokens whose importance ratios fall outside the trust region, CISPO \citep{chen2025minimax} introduces a soft-clipping mechanism that preserves all token-level gradients while constraining only the magnitude of the importance-weighting factor. In this paper, we develop an RL objective grounded in the prefix importance ratio. While GSPO \citep{zheng2025group} replaces the token-level ratio with a full-sequence importance ratio, our approach neither discards token-level information nor relies on the complete sequence ratio. Instead, we focus on the prefix ratio, which provides a more flexible and fine-grained correction signal.

\paragraph{Off-Policy RL for LLMs} Off-policy training, where rollouts are sampled using an older policy version and then used to update a newer policy, is common in LLM post-training. Using stale rollouts brings notable efficiency benefits. First, generating a large batch of rollouts at once reduces GPU idle time compared with repeatedly generating many smaller batches \citep{team2025kimi,gao2025rollpacker}. Second, asynchronous execution between inference (rollout generation) and training (gradient updates) further improves overall throughput \citep{noukhovitch2024asynchronous,roux2025tapered,fu2025areal,sheng2025laminar}. However, excessive data staleness can cause severe instability or even full training collapse. To address this challenge, several recent methods aim to stabilize training under large off-policy drift. BAPO~\citep{xi2025bapo} and M2PO~\citep{zheng2025prosperity} instead modify the clipping mechanism to dynamically filter unstable tokens. AsymRE \citep{arnal2026asymmetric} stabilizes off-policy RL by asymmetrically balancing positive and negative rewards. Its key idea is to lower the reward baseline, so that positive samples obtain larger relative advantages and therefore contribute more strongly to the policy-gradient update. In contrast, our work adopts a principled approach grounded in theoretical analysis and leverages prefix importance ratio information to mitigate instability under off-policy optimization.

\section{Preliminaries}

Given a question or prompt $\bm{q}$, an LLM $\pi_\theta$, parameterized by $\theta$, generates a response $\bm{o} = (o_1, o_2, \ldots, o_T)$ in an autoregressive manner, where each token $o_t$ is sampled according to $\pi_\theta(o_t \mid \bm{q}, \bm{o}_{<t})$. The training dataset is given by $\mathcal{S} = \{(\bm{q}_i, \bm{a}_i)\}_{i=1}^m$, where $\bm{a}_i$ denotes the ground truth final answer to $\bm{q}_i$ without any intermediate reasoning steps. Although $\bm{a}_i$ does not include a full reasoning chain, it still provides a reliable supervision signal because the correctness of a generated response can be verified by comparing its final answer with $\bm{a}_i$. Many contemporary LLM RL post-training methods are built upon the PPO clip objective \citep{schulman2017proximal}, which applies token-level clipping to control updates according to positive and negative advantages. PPO relies on a learned critic to estimate token advantages, while recently developed critic-free algorithms such as GRPO \citep{shao2024deepseekmath} adopt multi-sample Monte Carlo estimation instead of value function learning. In this paper, we focus on the critic-free paradigm and present a detailed exposition of two representative RL objectives within this family, namely the hard-clipping method GRPO and the soft-clipping approach CISPO.

\vspace{1mm}

\textbf{GRPO} estimates token advantages by sampling multiple rollouts per prompt. Its objective is defined as
\begin{align}
\mathcal{J}_{\mathrm{GRPO}}(\theta)=  \mathbb{E}_{\bm{q} \sim \mathcal{S},\left\{\bm{o}^i\right\}_{i=1}^G \sim \pi_{\theta_{\mathrm{old}}}(\cdot \mid \bm{q})} 
 {\left[\frac{1}{G} \sum_{i=1}^G \frac{1}{\left|\bm{o}^i\right|} \sum_{t=1}^{\left|\bm{o^i}\right|}\min \left(\rho_t^i(\theta) \hat{A}_t^i, \operatorname{clip}\left(\rho_t^i(\theta), 1-\varepsilon, 1+\varepsilon\right) \hat{A}_t^i\right)\right] }, \nonumber
\end{align}
where $\pi_{\theta_{\mathrm{old}}}$ is the sampling policy and $\rho_t^i(\theta)=\frac{\pi_\theta\left(o_t^i \mid \bm{q}, \bm{o}_{<t}^i\right)}{\pi_{\theta_{\text {old }}}\left(o_t^i \mid \bm{q}, \bm{o}_{<t}^i\right)}$ denotes the importance sampling ratio. After generating $G$ responses for each prompt $\bm{q}$, the token advantage is estimated as $\hat{A}_t^i=\frac{R^i-\operatorname{mean}\left(\left\{R^j\right\}_{j=1}^G\right)}{\operatorname{std}\left(\left\{R^j\right\}_{j=1}^G\right)}$, where $R^i$ is the reward assigned to the response $\bm{o}^i$ and is typically computed by comparing the predicted final answer with the ground truth.

\vspace{1mm}

\textbf{CISPO \quad}  In contrast to hard-clipping methods such as PPO and GRPO, which eliminate tokens whose importance ratios fall outside the trust region, CISPO \citep{chen2025minimax} adopts a soft-clipping strategy that retains all token-level gradients and constrains only extreme ratio values, thereby controlling gradient magnitudes rather than masking gradients. The CISPO objective is given by
\begin{align}
\mathcal{J}_{\mathrm{CISPO}}(\theta) =  & \mathbb{E}_{q \sim \mathcal{S},\left\{\bm{o}^i\right\}_{i=1}^G \sim \pi_{\theta_{\text {old }}}(\cdot \mid q)} \nonumber \\ & {\left[\frac{1}{\sum_{i=1}^G\left|\bm{o}^i\right|} \sum_{i=1}^G \sum_{t=1}^{\left|\bm{o}^i\right|} \operatorname{sg}\left(\operatorname{clip}\left(\rho_t^i(\theta), 1-\varepsilon_{\text {low }}, 1+\varepsilon_{\text {high }}\right)\right) \hat{A}_t^i \log \pi_\theta\left(o_t^i \mid \bm{q}, \bm{o}_{<t}^i\right)\right] }, \nonumber
\end{align}
where $\operatorname{sg}(\cdot)$ denotes the stop gradient operator, ensuring that clipping affects only the gradient magnitude.

\section{Method}

Most LLM RL post-training methods follow the PPO formulation and therefore use the token-level importance ratio $\rho_t=\frac{\pi_\theta\left(o_t \mid \bm{o}_{<t}\right)}{\pi_{\theta_{\text {old}}}\left(o_t \mid  \bm{o}_{<t}\right)}$, where we omit the prompt $\bm{q}$ for simplicity. In this section, {\it we take a step back and revisit how importance sampling ratios arise in the RL objective}, showing that the commonly used token-level ratio is an inaccurate approximation that leads to unstable optimization when training LLMs under large off-policy conditions.

\subsection{Theoretical Analysis}
\label{sec:thm}
For an LLM $\pi_\theta$, recall the generated response $\bm{o}=(o_1,o_2,...,o_T)$. Ignoring the prompt distribution for simplicity, the standard RL objective is given by the expected cumulative reward  $\mathcal{J}(\theta) \;=\; \mathbb{E}_{\bm{o} \sim \pi_\theta}[R(\bm{o})]$. Then the policy gradient $\nabla_\theta \mathcal{J}(\theta)$ is obtained via the following theorem.
\begin{lemma}[Policy Gradient Theorem \citep{sutton1999policy}]
\label{thm:pgt}
Let $\bm{o}= (o_1, o_2, \ldots, o_T)$ denote a trajectory generated by $\pi_\theta$.  
The gradient of the RL objective satisfies
\begin{equation}
\nabla_\theta \mathcal{J}(\theta)
=
\sum_{t=1}^T
\mathbb{E}_{(o_1,...,o_t)\sim \pi_\theta}
\Big[
    \nabla_\theta \log \pi_\theta(o_t \mid \bm{o}_{<t}) \, A^\pi(o_t;\bm{o}_{<t})
\Big], \nonumber
\end{equation}
where $\bm{o}_{< t} = (o_1,..., o_{t-1})$ is the prefix prior to step $t$ and  
$A^\pi(o_t;\bm{o}_{<t})$ is the advantage under $\pi_\theta$.
\end{lemma}
We provide a proof of this lemma in Appendix \ref{app:proof} under an LLM-specific autoregressive setting. Based on \Cref{thm:pgt}, we can directly derive the policy gradient in the off-policy setting, where the response is generated by a separate sampling policy $\pi_{\theta_{\mathrm{old}}}$.
\begin{theorem}[Policy Gradient under Off-policy Conditions]
\label{corollary:main}
Suppose trajectories are sampled from an older policy $\pi_{\theta_{\mathrm{old}}}$ and used to optimize the current policy $\pi_\theta$.
Let the prefix importance ratio be
\[
\rho_{1:t}
=
\frac{ \mathbb{P}_\theta(o_1, \ldots, o_t)}{\mathbb{P}_{\theta_{\mathrm{old}}}(o_1, \ldots, o_t)}
=
\prod_{i=1}^t
\frac{
    \pi_\theta(o_i \mid \bm{o}_{<i})
}{
    \pi_{\theta_{\mathrm{old}}}(o_i \mid \bm{o}_{<i})
}
=
\prod_{i=1}^t \rho_i,
\]
where $\mathbb{P}_\theta(o_1,...,o_t)$ denotes the probability of generating the sequence $(o_1, ..., o_t)$ under $\pi_\theta$. Then the policy gradient under off-policy sampling satisfies
\begin{equation}
\nabla_\theta \mathcal{J}(\theta)
=
\sum_{t=1}^T
\mathbb{E}_{(o_1,...,o_t) \sim \pi_{\theta_{\mathrm{old}}}}
\Big[
    \rho_{1:t}\,
    \nabla_\theta \log \pi_\theta(o_t \mid \bm{o}_{<t})\,
    A^\pi(o_t;\bm{o}_{<t}) 
\Big]. \nonumber
\end{equation}
\end{theorem}
As shown in \Cref{corollary:main}, the prefix importance ratio $\rho_{1:t}$ provides a theoretically rigorous correction for distribution shift under off-policy regimes, whereas existing approaches rely on the relaxed token-level ratio $\rho_t$. To illustrate the consequences of this approximation, we compare the training dynamics of three widely used methods, GRPO, GSPO, and CISPO, under both on-policy and off-policy settings in \Cref{figure:training collapse}. For hard-clipping methods such as GRPO and GSPO, operating in an off-policy regime results in substantially lower rewards with severe oscillations, accompanied by similarly unstable entropy dynamics, ultimately leading to degraded performance. In contrast, the soft-clipping method CISPO exhibits an even more pronounced failure mode, with training collapsing in both on-policy and off-policy settings as rewards abruptly drop and entropy explodes to extremely large values. In sum, while relaxing the prefix importance ratio $\rho_{1:t}$ to its token-level approximation $\rho_t$ is often workable in on-policy regimes, relying on the token-level ratio leads to unstable optimization and inferior or even collapsed training under off-policy conditions.

By contrasting the relaxed token-level ratio $\rho_t$ with the full prefix ratio $\rho_{1:t}$, we observe that in autoregressive LLM environments, where $\rho_{1:t} = \prod_{i=1}^t \rho_i$ grows multiplicatively with sequence length, the token-level approximation rapidly diverges from the true prefix ratio as off-policyness increases. Moreover, LLM-generated rollouts are often very long, frequently exceeding $10{,}000$ tokens, which further amplifies this divergence. This growing mismatch ultimately leads to unstable optimization dynamics and training collapse in large off-policy regimes. As a result, the commonly used token-level ratio $\rho_t$ becomes increasingly unreliable under significant off-policy drift.

\begin{figure}[t]
\centering
\subfigure[Rewards vs. Steps]{
\begin{minipage}[b]{\textwidth}
\centering
\includegraphics[width=0.325\columnwidth]{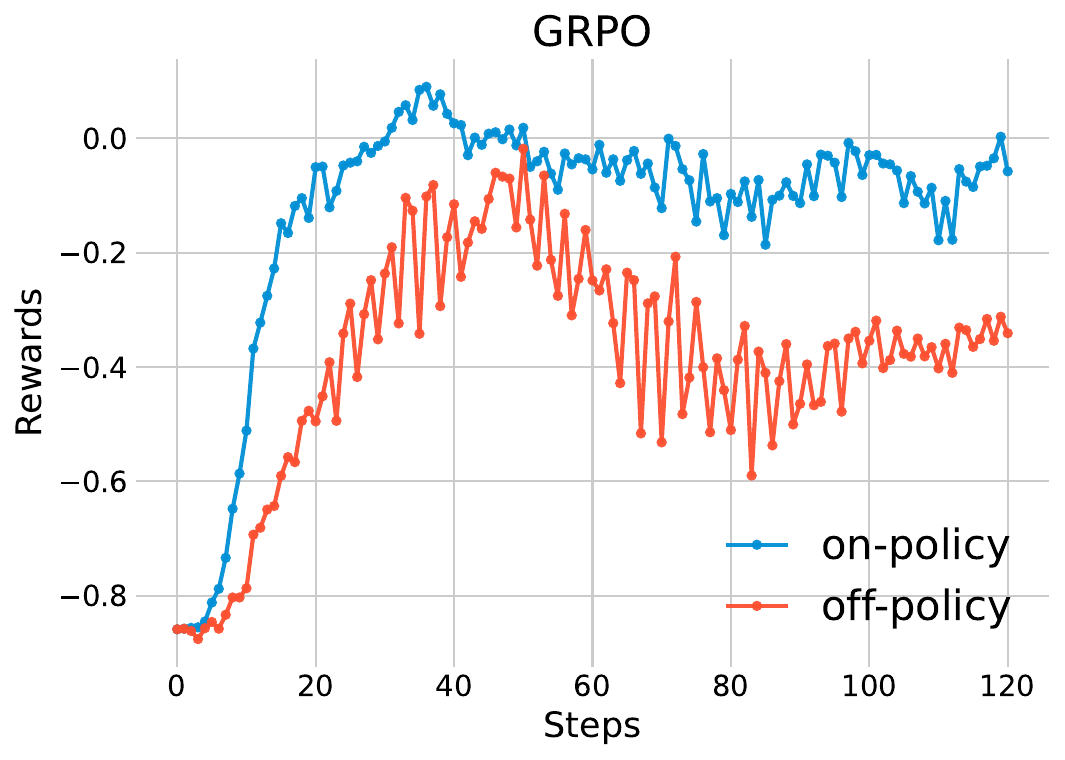}
\hfill
\includegraphics[width=0.325\columnwidth]{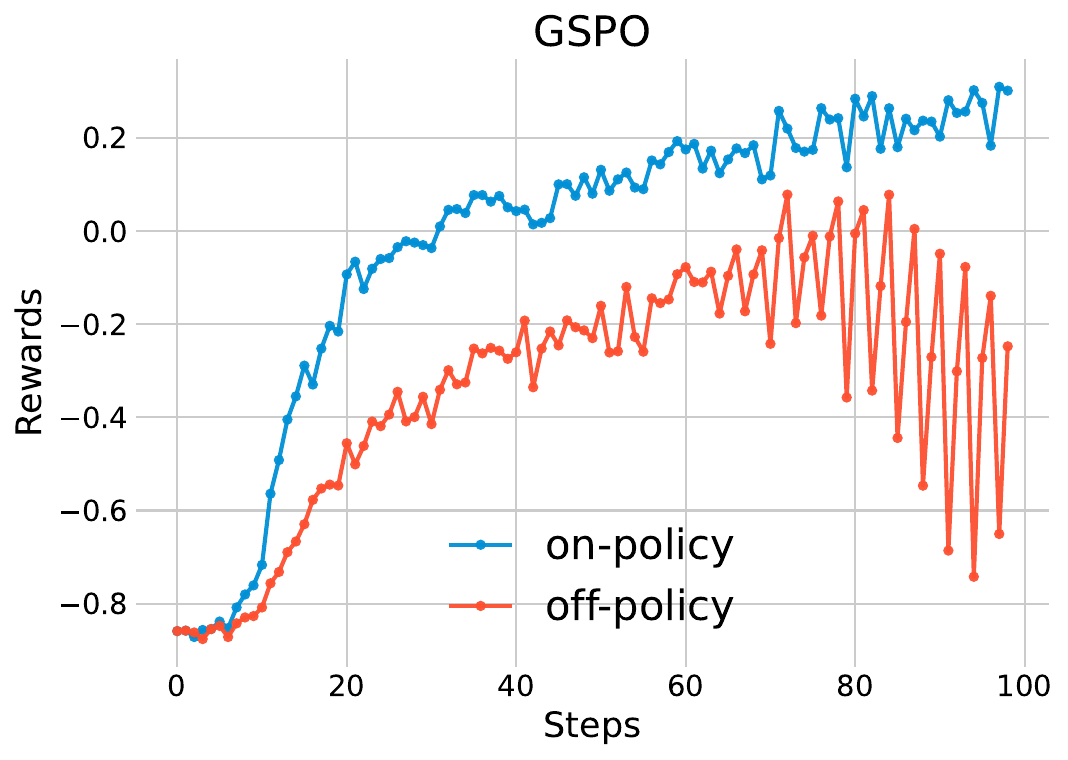}
\hfill
\includegraphics[width=0.325\columnwidth]{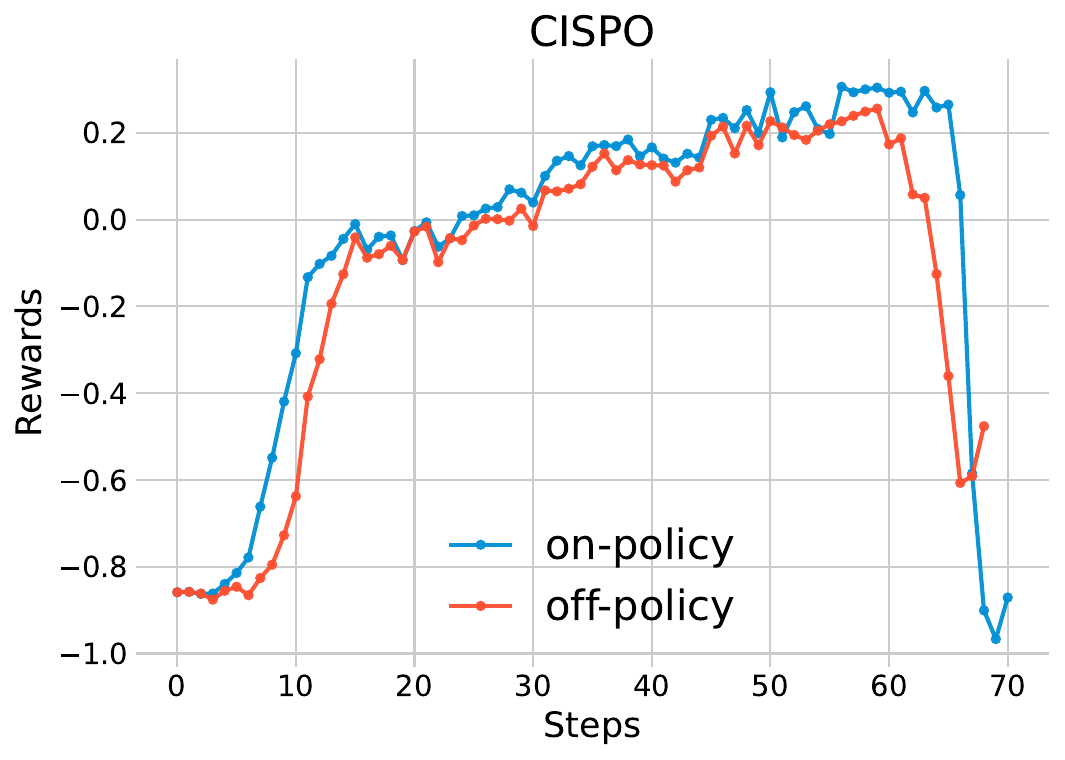}
\end{minipage}
\label{figure:step vs rewards}  
}
\subfigure[Entropy vs. Steps]{
\begin{minipage}[b]{\textwidth}
\centering
\includegraphics[width=0.325\columnwidth]{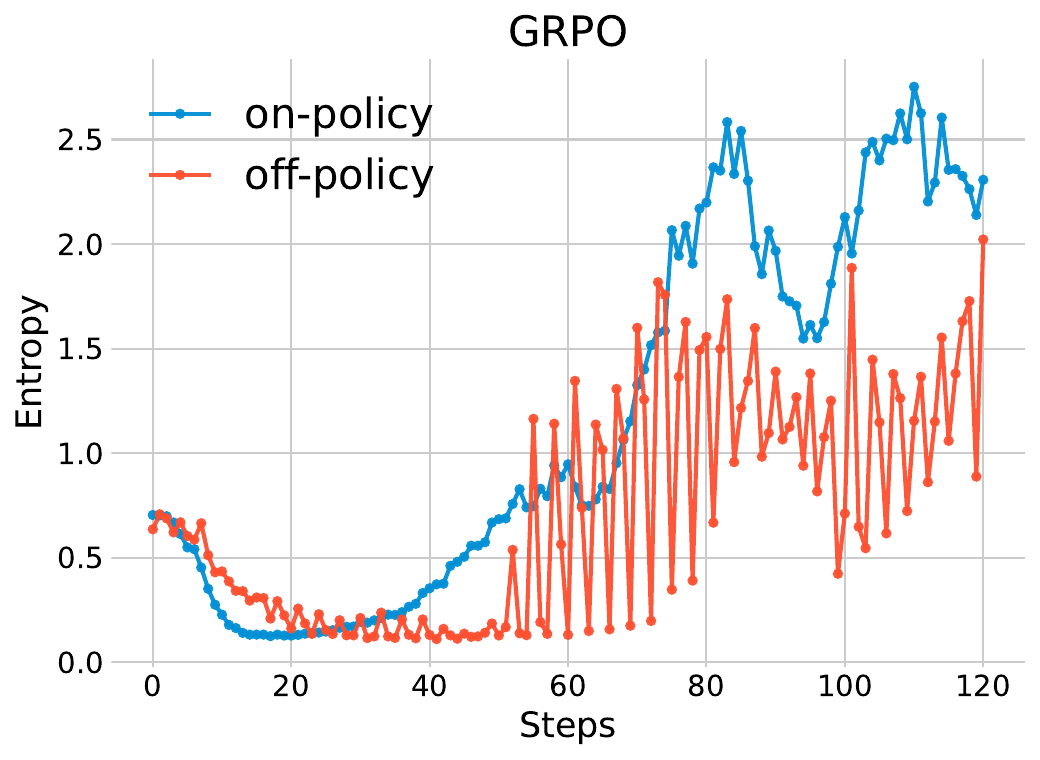}
\hfill
\includegraphics[width=0.325\columnwidth]{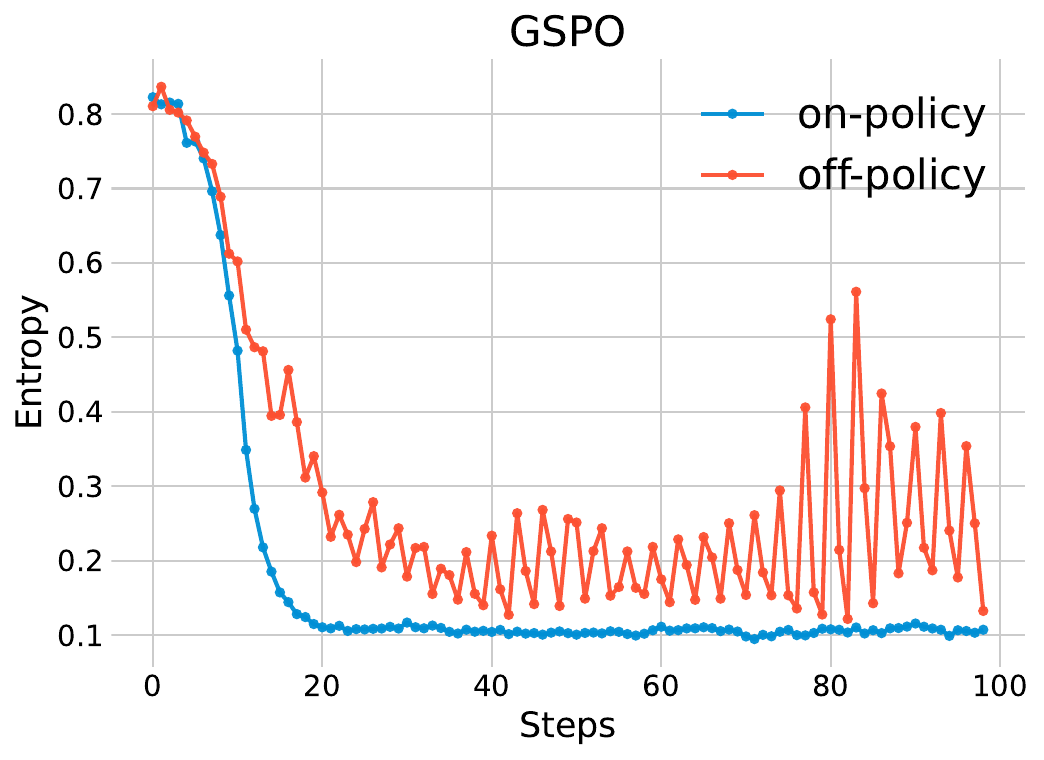}
\hfill
\includegraphics[width=0.325\columnwidth]{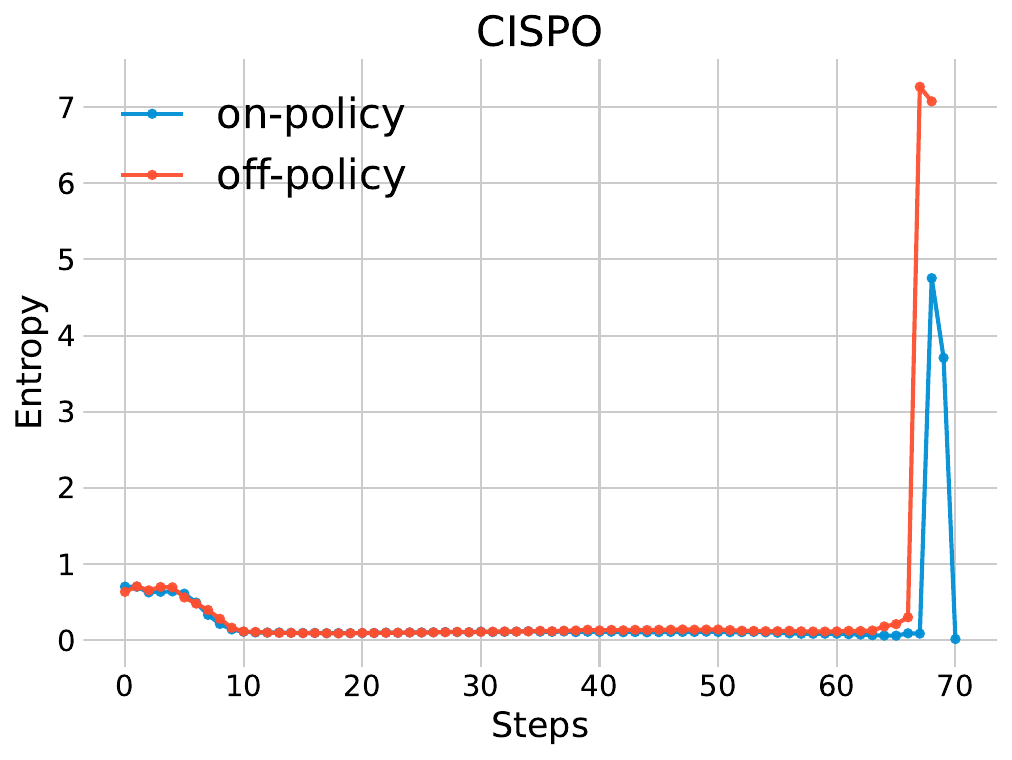}
\end{minipage}
\label{figure:step vs entropy}  
}
\caption{Plots of (a) reward and (b) entropy as functions of training steps for Qwen3-8B-Base under on-policy and off-policy regimes. In the off-policy setting, each batch of sampled rollouts is stored in a buffer and used for training after a delay of $2$ global steps, whereas the on-policy setting applies no delay between sampling and optimization.}
\label{figure:training collapse}
\end{figure}

\subsection{MinPRO}

To achieve stable LLM optimization, a more principled treatment should explicitly incorporates the information carried by the prefix ratio $\rho_{1:t}$. A naive approach would be to use $\rho_{1:t}$ directly, possibly with heuristics such as upper clipping to prevent numerical explosion. However, since the prefix ratio is a cumulative product, it suffers from two fundamental limitations. First, it is prone to extreme values, resulting in large variance. Second, these extreme values tend to arise near the end of the generated sequence, introducing a strong length bias. This bias leads to inconsistent updates across token positions and ultimately undermines effective policy optimization.

To incorporate prefix ratio information while eliminating the large variance and length bias induced by cumulative token ratio products, we propose a simple yet stable proxy termed \textit{Minimum Prefix Ratio} (MinPRO). Instead of using the full prefix ratio $\rho_{1:t} = \rho_{1:t-1} \cdot \rho_t$, we consider only the smallest token ratio that appears prior to step $t$:
\[
{\color{black}
\underline{\rho}_t = \min_{i < t} \rho_i}.
\]
We then approximate the prefix ratio as $\rho_{1:t} \approx \underline{\rho}_t \cdot \rho_t$. This formulation replaces the unstable cumulative product with a non-cumulative and numerically stable surrogate that preserves essential prefix-level correction signals while mitigating variance and sequence-length-induced artifacts. We therefore formulate the MinPRO objective as follows:
\begin{align}
\mathcal{J}_{\mathrm{MinPRO}}(\theta) =  & \mathbb{E}_{q \sim \mathcal{S},\left\{\bm{o}^i\right\}_{i=1}^G \sim \pi_{\theta_{\text {old }}}(\cdot \mid q)} \nonumber \\ & {\left[\frac{1}{\sum_{i=1}^G\left|\bm{o}^i\right|} \sum_{i=1}^G \sum_{t=1}^{\left|\bm{o}^i\right|} \operatorname{sg}\left(\operatorname{clip}\left({\color{magenta}\underline{\rho}_t^i}  \rho_t^i, 1-\varepsilon_{\text {low }}, 1+\varepsilon_{\text {high }}\right)\right) \hat{A}_t^i \log \pi_\theta\left(o_t^i \mid \bm{q}, \bm{o}_{<t}^i\right)\right] } \nonumber
\end{align}
As shown in the MinPRO formulation, the key modification is to multiply the token ratio $\rho_t$ by the minimum prefix ratio $\underline{\rho}_t$. Intuitively, when the prefix ratio $\rho_{1:t-1}$ becomes extremely small, it indicates that the prefix fragment is unlikely to be sampled under the current policy. In this case, the gradient associated with token $o_t$ should not exert a strong influence on policy optimization, even if the current token ratio $\rho_t$ remains within a normal range. However, prior approaches ignore prefix-level information and still apply this gradient, which can lead to unstable updates under large off-policy drift. MinPRO alleviates this issue by incorporating the simple yet effective factor $\underline{\rho}_t$, which suppresses the contribution of such tokens when the prefix suggests that the trajectory is unlikely under the current policy, thereby stabilizing optimization in off-policy regimes.

\section{Experiments}
\label{sec:expriments}

\textbf{Datasets and Models \ \ } We perform RL post-training on the DAPO-Math-17K dataset \citep{yu2025dapo}, which consists of $17{,}000$ mathematical questions, each paired with a final integer answer. As base models, we use two dense LLMs, Qwen3-8B-Base and Qwen3-14B-Base, as well as an MoE model, Qwen3-30B-A3B-Base \citep{yang2025qwen3}. These models are widely adopted pre-trained LLMs that have not undergone instruction tuning or reasoning-specific training, making them well-suited testbeds for evaluating post-training algorithms designed to elicit reasoning capabilities. During post-training, we evaluate mathematical reasoning performance on seven benchmarks of AMC23, AIME24, AIME25, MATH500, Olympiad, Minerva, and GSM8K.

\vspace{1mm}

\textbf{Setup \ \ } We conduct LLM RL post-training using the VeRL framework \citep{sheng2025hybridflow} and evaluate performance with \texttt{pass@k}, which measures the probability of obtaining at least one correct response among $k$ sampled attempts. We adopt a critic-free setup, computing token-level advantages from normalized rollout rewards. The maximum response length is $20{,}480$ tokens, and KL regularization is disabled. For Qwen3-8B-Base and Qwen3-30B-A3B-Base, we use a batch size of $512$ and a mini-batch size of $32$. For Qwen3-14B-Base, we reduce these sizes to $256$ and $16$, respectively, to avoid out-of-memory issues. Both configurations yield $16$ parameter update steps per global training step. Qwen3-8B-Base, Qwen3-14B-Base, and Qwen3-30B-A3B-Base are trained for $120$, $160$, and $120$ global steps, respectively, which is sufficient for convergence. Checkpoints are saved every $10$ global steps, and we select the checkpoint with the highest average score across the evaluation benchmarks. Additional implementation details are provided in Appendix~\ref{app:implementation details}.

To evaluate training stability, we explicitly introduce a large off-policy regime by controlling policy lag through a rollout buffer. Specifically, each batch of sampled rollouts is stored in the buffer and used for training after a delay of $n$ global steps. Under this setup, the number of parameter updates separating the behavior policy $\pi_{\theta_{\mathrm{old}}}$ and the current policy $\pi_\theta$ ranges from $16n$ to $16(n+1)-1$. In our experiments, we set the data staleness to $n=2$ unless otherwise specified, corresponding to a highly off-policy regime that induces unstable optimization for many existing algorithms.

\vspace{1mm}

\textbf{Baselines \ \ } We evaluate several representative baseline methods along with their corresponding hyperparameters: (1) {GRPO} \citep{shao2024deepseekmath,yu2025dapo}: we set $\epsilon_{\mathrm{low}} = 0.2$ and $\epsilon_{\mathrm{high}} = 0.28$; (2) GSPO \citep{zheng2025group}: we set $\epsilon_{\mathrm{low}} = \epsilon_{\mathrm{high}}=2e-3$; (3) {CISPO} \citep{chen2025minimax}: we use $\epsilon_{\mathrm{low}} = 1$ and $\epsilon_{\mathrm{high}} = 4$; (4) {M2PO} \citep{zheng2025prosperity}: a GRPO-style hard-clipping method designed for large off-policy regimes, with budget $M_2 = 0.04$ following the original paper; (5) {MinPRO}: our proposed method, which adopts the same settings $\epsilon_{\mathrm{low}} = 1$ and $\epsilon_{\mathrm{high}} = 4$ as CISPO.

\begin{figure}[t]
\centering

\includegraphics[width=0.493\columnwidth]{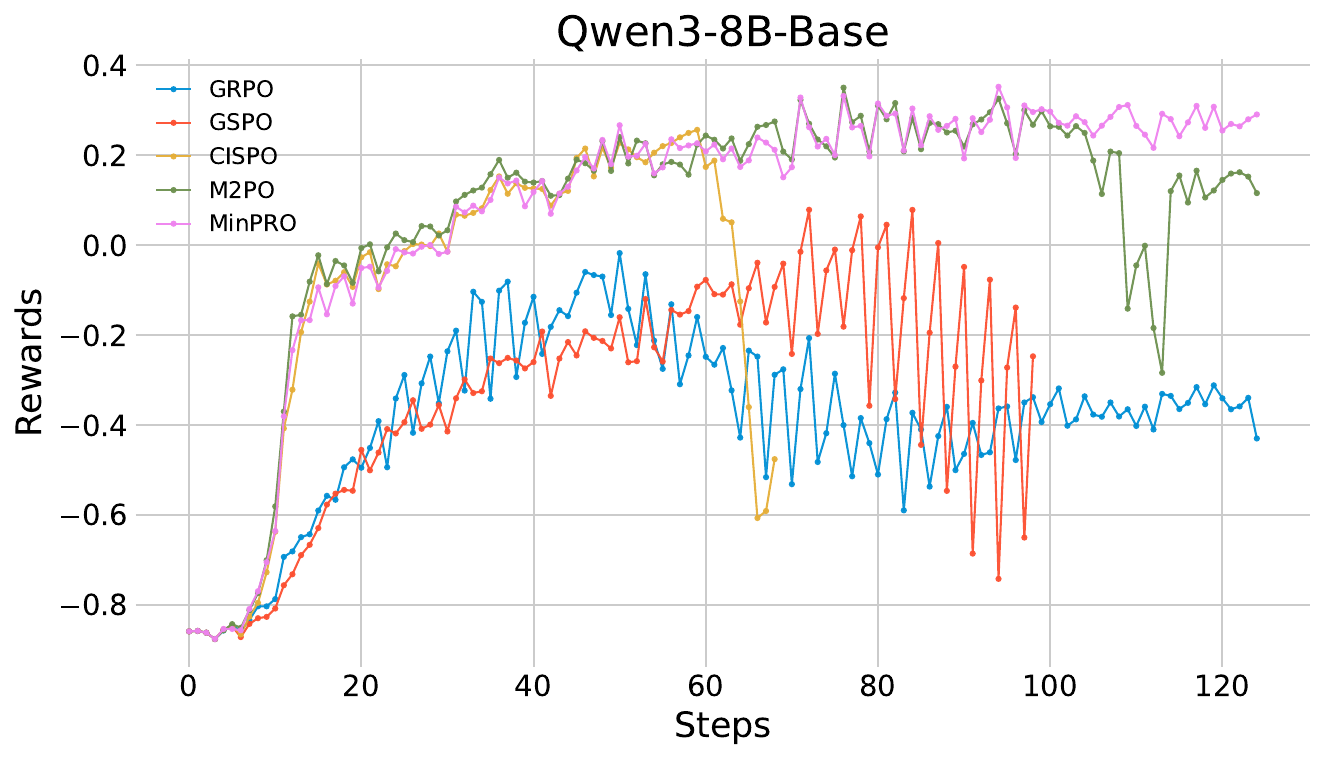}
\hfill
\includegraphics[width=0.493\columnwidth]{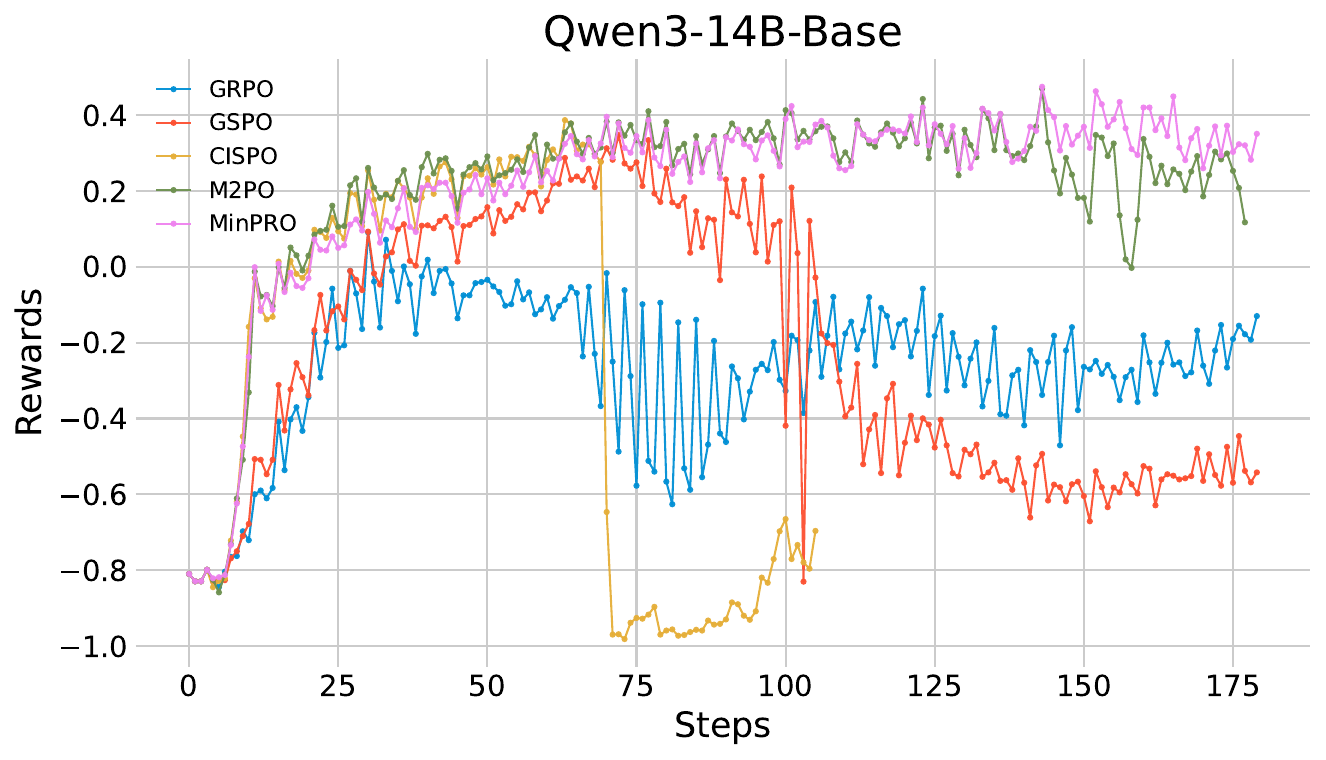}
\caption{Training reward curves as functions of training steps for Qwen3-8B-Base and Qwen3-14B-Base under off-policy optimization.}
\label{figure:rewards}
\end{figure}

\begin{table}[t]
\centering
\caption{\texttt{pass@1} scores under large off-policyness ($n=2$).}
\vspace{1mm}
\label{tab:pass1-s2}
\setlength{\tabcolsep}{6.4pt}{
\begin{tabular}{cc cccc  cccc}
\toprule
& {Method} & {\footnotesize AMC23} &
{\footnotesize AIME24} & {\footnotesize AIME25} &
{\footnotesize MATH500} & {\footnotesize Olympiad} &
{\footnotesize Minerva} & {\footnotesize GSM8K} &
{Avg} \\
\specialrule{1pt}{0.2\jot}{0.15pc}
\cellcolor{blue!15}  & {Base} & $21.6$ & $3.5$ & $3.5$ & $44.5$ & $20.2$ & $20.6$  & $70.3$ & $26.3$\\
\cellcolor{blue!15} & {GRPO}  & $63.3$ & $23.1$ & $18.5$ & $70.2$ & $40.4$ & $29.0$  & $80.4$ & $46.4$\\
\cellcolor{blue!15} & {GSPO}  & $64.4$ & $25.9$ & $18.3$ & $78.4$ & $44.4$ & $31.4$  & $87.3$ & $50.0$\\
\cellcolor{blue!15} & {CISPO} & $78.9$ & {$35.9$} & $25.1$ & $84.5$ & $54.3$ & $32.2$  & $86.7$ & $56.8$\\
\cellcolor{blue!15} & {M2PO} & {$82.2$} & $33.9$ & {$25.8$} & $85.8$ & {\bm{$56.5$}} & {\bm{$37.7$}}  & {$88.7$} & {$58.7$}\\
\multirow{-5.7}{*}{\cellcolor{blue!15}\rotatebox{90}{\small 8B-Base}} & {MinPRO} & {\bm{$82.5$}} & {\bm{$36.1$}} & {\bm{$30.3$}} & {\bm{$86.5$}} & {\bm{$56.5$}} & {$33.6$}  & {\bm{$90.2$}} & {\bm{$59.4$}}\\
\midrule
\cellcolor{blue!15}  & {Base} & $22.9$ & $2.8$ & $2.9$ & $68.8$ & $35.6$ & $24.6$  & $68.4$ & $32.3$\\
\cellcolor{blue!15} & {GRPO}  & $67.0$ & $26.2$ & $23.1$ & $80.0$ & $47.8$ & $31.1$ & $86.2$ & $51.6$\\
\cellcolor{blue!15} & {GSPO}  & $80.8$ & $44.1$ & {$31.8$} & $85.4$ & $54.6$ & $33.5$ & \bm{$92.4$} & $60.4$\\
\cellcolor{blue!15} & {CISPO} & $85.7$ & $45.3$ & $32.7$ & $87.4$ & $58.0$ & \bm{$35.7$}  & $87.0$ & $61.7$\\
\cellcolor{blue!15} & {M2PO} & $85.7$ & $46.3$ & $32.0$ & $87.7$ & $57.3$ & $33.8$  & {$91.8$} & $62.1$\\
\multirow{-5.5}{*}{\cellcolor{blue!15}\rotatebox{90}{\small 14B-Base}} 
 & {MinPRO} & \bm{$87.5$} & \bm{$47.0$} & \bm{$33.4$} & \bm{$88.1$} & \bm{$58.8$} & $33.3$  & $90.1$ & \bm{$62.6$}\\
\bottomrule
\end{tabular}}
\end{table}

\begin{table}[t]
\centering
\caption{\texttt{pass@1} on Qwen3-8B-Base under more off-policy ($n=4$) and near-on-policy ($n=0$).}
\vspace{1mm}
\label{tab:qwen3_8b_n0_n4}
\setlength{\tabcolsep}{6.4pt}
{
\begin{tabular}{c c c c c c c c c c}
\toprule
& {Method} & {\footnotesize AMC23} &
{\footnotesize AIME24} & {\footnotesize AIME25} &
{\footnotesize MATH500} & {\footnotesize Olympiad} &
{\footnotesize Minerva} & {\footnotesize GSM8K} &
{Avg} \\
\midrule\multirow{5}{*}{\rotatebox{90}{$n=4$}}
& GRPO   & $60.5$ & $19.1$ & $15.2$ & $74.4$ & $43.6$ & $34.2$ & $82.2$ & $47.0$ \\
& GSPO   & $60.8$ & $20.7$ & $15.8$ & $73.6$ & $44.2$ & $34.2$ & $88.2$ & $48.2$ \\
& CISPO  & $78.6$ & $34.3$ & $25.4$ & $84.3$ & $53.0$ & $33.0$ & $87.3$ & $56.6$ \\
& M2PO   & $77.9$ & $34.8$ & $24.5$ & $83.8$ & $54.0$ & $\bm{36.0}$ & $87.6$ & $56.9$ \\
& MinPRO & $\bm{79.3}$ & $\bm{38.2}$ & $\bm{28.3}$ & $\bm{84.8}$ & $\bm{54.2}$ & $35.3$ & $\bm{88.3}$ & $\bm{58.3}$ \\
\midrule
\multirow{5}{*}{\rotatebox{90}{$n=0$}}
& GRPO   & $80.7$ & $33.1$ & $26.2$ & $84.8$ & $53.8$ & $\bm{36.0}$ & $90.4$ & $57.8$ \\
& GSPO   & $82.4$ & $\bm{37.9}$ & $29.0$ & $\bm{86.3}$ & $53.8$ & $35.3$ & $89.7$ & $59.2$ \\
& CISPO  & $\bm{82.5}$ & $36.1$ & $27.4$ & $85.5$ & $\bm{55.3}$ & $34.9$ & $88.9$ & $58.7$ \\
& M2PO   & $81.2$ & $37.8$ & $25.4$ & $85.3$ & $53.4$ & $\bm{36.0}$ & $\bm{90.6}$ & $58.5$ \\
& MinPRO & $\bm{82.5}$ & $\bm{37.9}$ & $\bm{29.5}$ & $85.5$ & $54.4$ & $35.1$ & $90.2$ & $\bm{59.3}$ \\
\bottomrule
\end{tabular}
}
\end{table}

\subsection{Main Results}
\textbf{Optimization Stability \ \ } To show the training stability of different LLM RL algorithms, we plot their training reward curves, where the reward at each step is computed as the average reward over the training batch. This quantity plays a role analogous to the (negative) training loss in supervised learning, and a steadily increasing reward curve therefore indicates a stable optimization process. As shown in \Cref{figure:rewards}, we report training reward curves under the large off-policy regime for both 8B and 14B LLMs. In this setting, GRPO, GSPO, and CISPO exhibit highly unstable reward dynamics, while MinPRO maintains consistently higher and more stable rewards throughout training than all other baselines, including the recent M2PO method. This suggests that MinPRO can support stable long-horizon training even under severe off-policy conditions.

\begin{table}[t]
\centering
\caption{Mean-dataset \texttt{pass@k} scores averaged over AMC23, AIME24, and AIME25 under large off-policyness ($n=2$). Boldface indicates the highest score within each model.}
\vspace{1mm}
\label{tab:passk avg}
\setlength{\tabcolsep}{7.7pt}
{
\begin{tabular}{c c cccccccc c}
\toprule
& \multirow{2}{*}{Method} & \multicolumn{8}{c}{\texttt{Pass@k}} & \multirow{2}{*}{Avg} \\
\cmidrule(lr){3-10}
 &  & $1$ & $2$ & $4$ & $8$ & $16$ & $32$ & $64$ & $128$ &  \\
\specialrule{1pt}{0.2\jot}{0.15pc}

\cellcolor{blue!15} & GRPO   & $35.0$ & $43.2$ & $51.1$ & $57.8$ & $63.4$ & $67.7$ & $71.3$ & $73.9$ & $58.0$ \\
\cellcolor{blue!15} & GSPO   &                      $36.2$ & $44.0$ & $51.4$ & $58.0$ & $63.5$ & $68.0$ & $71.2$ & $73.5$ & $58.2$ \\
\cellcolor{blue!15} & CISPO  &                        $46.6$ & $54.3$ & $60.5$ & $65.2$ & $69.2$ & $72.6$ & \bm{$75.7$} & \bm{$78.3$} & $65.3$ \\
\cellcolor{blue!15} & M2PO   &                        $47.3$ & $54.4$ & $60.7$ & $65.7$ & $69.5$ & $72.8$ & $75.6$ & $77.5$ & $65.4$ \\
\multirow{-4.7}{*}{\cellcolor{blue!15}\rotatebox{90}{\small 8B-Base}} & MinPRO &                        \bm{$49.6$} & \bm{$56.7$} & \bm{$62.5$} & \bm{$67.0$} & \bm{$70.5$} & \bm{$73.1$} & $75.3$ & $77.2$ & \bm{$66.5$} \\
\midrule

\cellcolor{blue!15} &GRPO    & $38.8$ & $47.3$ & $54.7$ & $60.7$ & $65.7$ & $69.6$ & $72.8$ & $75.2$ & $60.6$ \\
\cellcolor{blue!15} & GSPO   &                        {$52.2$} & {$61.5$} & \bm{$68.0$} & \bm{$72.2$} & {$75.2$} & {$77.6$} & {$79.5$} & {$81.0$} & {$70.9$} \\
\cellcolor{blue!15} & CISPO  &                        $54.6$ & $62.0$ & $67.0$ & $71.2$ & $75.0$ & $78.3$ & $81.2$ & $83.4$ & $71.6$ \\
\cellcolor{blue!15} &M2PO   &                        $54.7$ & $61.1$ & $65.7$ & $69.8$ & $73.5$ & $76.7$ & $79.6$ & $82.1$ & $70.4$ \\
\multirow{-4.5}{*}{\cellcolor{blue!15}\rotatebox{90}{\small 14B-Base}} & MinPRO &                        \bm{$56.0$} & \bm{$63.0$} & \bm{$68.0$} & $72.0$ & \bm{$75.5$} & \bm{$78.9$} & \bm{$82.4$} & \bm{$85.4$} & \bm{$72.6$} \\
\bottomrule
\end{tabular}
}
\end{table}

\begin{figure}[t]
\centering
\subfigure[Rewards vs. Steps]{
\label{figure:30b-rewards}
\includegraphics[width=0.478\columnwidth]{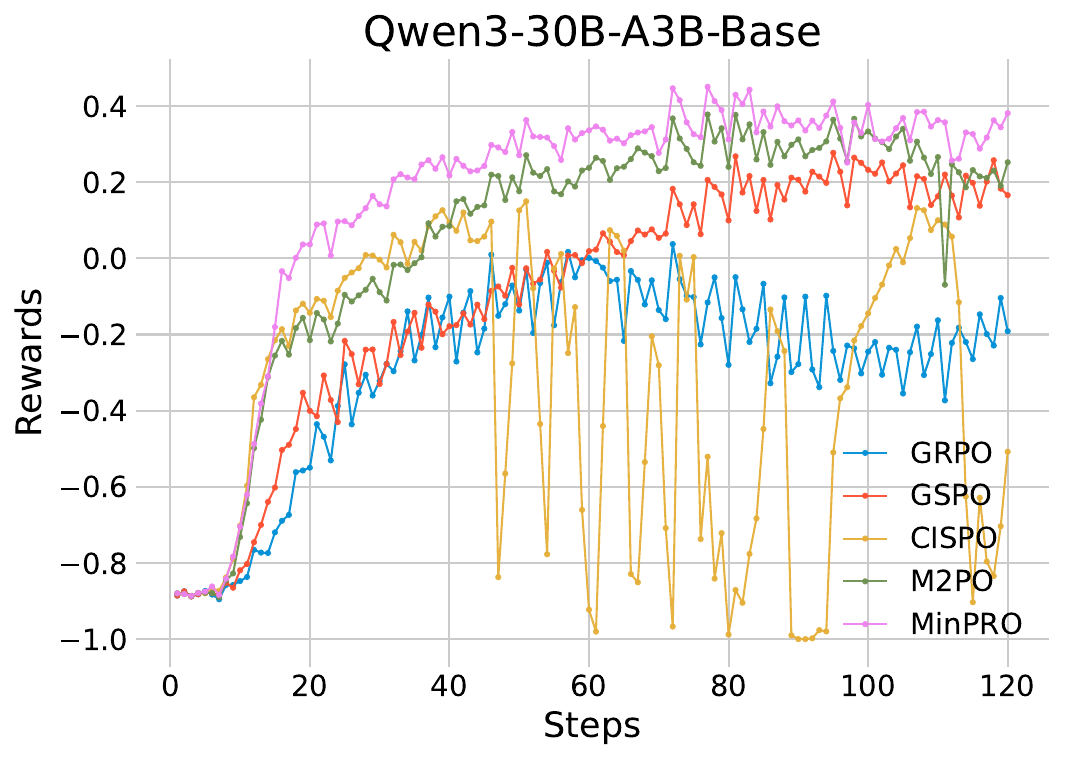}
}
\hfill
\subfigure[Average \texttt{pass@1} Scores vs. Steps]{
\label{figure:30b-avg}
\includegraphics[width=0.478\columnwidth]{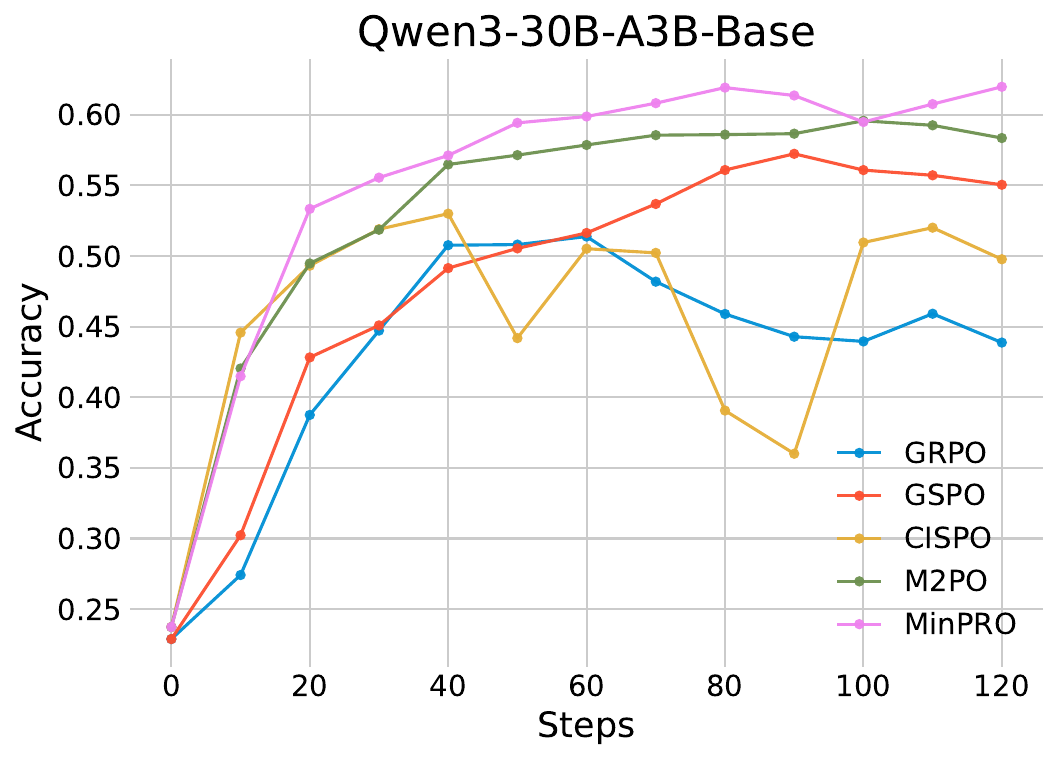}
}
\caption{(a) Training reward curves and (b) average \texttt{pass@1} scores as functions of training steps for Qwen3-30B-A3B-Base under off-policy optimization.}
\label{figure:30b}
\end{figure}

\vspace{1mm}

\textbf{Pass@1 Results\ \ }
For AMC23, AIME24, and AIME25, \texttt{pass@1} is estimated using $128$ sampled generations per prompt, while the other four larger benchmarks are evaluated with $2$ generations per prompt. The results are summarized in \Cref{tab:pass1-s2}. As shown, GRPO suffers from inferior performance under large off-policy drift, likely due to unstable training dynamics. In contrast, MinPRO achieves the highest average \texttt{pass@1} scores, outperforming all baselines by at least $0.5$ points on both the 8B and 14B models. These results demonstrate the effectiveness of MinPRO for off-policy optimization.

We further evaluate MinPRO under different degrees of off-policyness, including a more off-policy setting ($n=4$) and the near-on-policy setting ($n=0$); results are reported in \Cref{tab:qwen3_8b_n0_n4}. When $n=4$, MinPRO again achieves the best average performance, outperforming the strong off-policy baseline M2PO by $1.4$ points. Moreover, when $n=0$, MinPRO also attains the best average \texttt{pass@1} score, indicating that it remains highly competitive in the near-on-policy regime without sacrificing performance relative to standard on-policy-style optimization. Overall, these results suggest that MinPRO is effective across a broad range of off-policy regimes.

\vspace{1mm}

\textbf{Pass@k Results \ \ } To further assess performance beyond \texttt{pass@1}, we report average \texttt{pass@k} scores on AMC23, AIME24, and AIME25 in \Cref{tab:passk avg}. MinPRO achieves the best overall \texttt{pass@k} performance on both Qwen3-8B-Base and Qwen3-14B-Base, outperforming all baselines by at least $1$ point. 

\subsection{Scaling to Large MoE LLMs}

To further assess the effectiveness of MinPRO, we additionally conduct post-training on the widely used MoE LLM Qwen3-30B-A3B-Base, which contains $30$ billion total parameters while activating $3$ billion parameters during inference. We also perform off-policy training by setting the data staleness to $2$. The results are summarized in \Cref{figure:30b}. As shown in \Cref{figure:30b-rewards}, MinPRO maintains consistently higher and more stable training rewards than all other baselines throughout training, consistent with the trends observed on the 8B and 14B dense models. For evaluation, we report the averaged \texttt{pass@1} accuracy across the seven benchmarks over the course of training in \Cref{figure:30b-avg}, with detailed per-dataset curves provided in \Cref{figure:30b pre-data scores} in Appendix \ref{app:tables}. MinPRO again achieves superior and more stable performance compared to all baseline methods. These results demonstrate that MinPRO scales effectively to large MoE models and provides stable policy optimization.

\subsection{Additional Comparisons and Ablations}
\label{sec:main-additional-comparisons}

\begin{table}[t]
\centering
\caption{Additional \texttt{pass@1} comparisons (\%) on Qwen3-8B-Base. Avg is the mean over seven benchmarks. The shared MinPRO reference and groups $(a)$--$(b)$ use $n=4$; group $(c)$ varies staleness. The shared reference is reported once.}
\vspace{1mm}
\label{tab:combined-comparisons}
\label{tab:prefix-statistics}
\label{tab:sequence-baselines}
\label{tab:extreme-staleness}
\setlength{\tabcolsep}{4pt}
{
\begin{tabular}{@{}l*{8}{c}@{}}
\toprule
Method / Statistic & {\small AMC23} &
{\small AIME24} & {\small AIME25} &
{\small MATH500} & {\small Olympiad} &
{\small Minerva} & {\small GSM8K} & Avg \\
\midrule
\rowcolor{blue!10}
MinPRO (minimum) & 79.3 & 38.2 & 28.3 & 84.8 & 54.2 & 35.3 & 88.3 & \textbf{58.3} \\
\midrule
\multicolumn{9}{@{}l}{\textit{(a) Alternative prefix statistics}} \\
Arithmetic mean & 72.4 & 29.8 & 21.9 & 83.0 & 51.6 & 33.5 & 86.1 & 54.0 \\
Geometric mean & 78.9 & 36.5 & 27.1 & 85.4 & 51.8 & 34.9 & 86.4 & 57.3 \\
Cumulative prod & 68.6 & 27.7 & 21.6 & 82.6 & 50.0 & 32.4 & 85.7 & 52.7 \\
\midrule
\multicolumn{9}{@{}l}{\textit{(b) Recent baselines}} \\
AsymRE & 59.5 & 17.3 & 13.3 & 74.6 & 42.6 & 35.3 & 87.8 & 47.2 \\
VESPO & 74.6 & 31.1 & 24.1 & 82.0 & 53.1 & 36.0 & 85.1 & 55.1 \\
SNIB & 65.5 & 24.3 & 19.1 & 76.8 & 45.3 & 28.3 & 84.5 & 49.1 \\
\midrule
\multicolumn{9}{@{}l}{\textit{(c) MinPRO at larger staleness}} \\
MinPRO ($n=8$) & 78.9 & 35.6 & 27.1 & 85.0 & 52.6 & 34.6 & 87.0 & 57.3 \\
MinPRO ($n=16$) & 78.2 & 35.9 & 27.3 & 83.8 & 52.5 & 34.6 & 88.3 & 57.2 \\
\bottomrule
\end{tabular}
}
\end{table}

\paragraph{Choice of Prefix Statistic}
To assess the choice of the minimum prefix statistic $\underline{\rho}_t \rho_t$ in MinPRO, we compare it with three alternative ways of aggregating prefix ratios: the arithmetic mean $\frac{1}{t}\sum_{i=1}^{t}\rho_i$, geometric mean $(\prod_{i=1}^{t}\rho_i)^{1/t}$, and cumulative product $\prod_{i=1}^{t}\rho_i$, while keeping all other settings fixed.
As shown in Table~\ref{tab:combined-comparisons}(a), the minimum achieves the highest average \texttt{pass@1} score ($58.3$), outperforming all three alternatives.
Here we explain that the minimum preserves a bottleneck signal: a low preceding token ratio indicates a local mismatch between the current and behavior policies.
Whereas averaging can dilute this signal, the minimum retains it to downweight subsequent updates.
It also avoids cumulative multiplication, which can produce extreme ratios and cause clipping saturation.
These properties offer an intuitive explanation for MinPRO's empirical advantage.

\vspace{1mm}

\textbf{More Baselines \ \ }
We further compare MinPRO with recent baselines, including AsymRE~\citep{arnal2026asymmetric} (see Appendix~\ref{app:asymre} for adaptation details), VESPO~\citep{shen2026vespo}, and SNIB~\citep{shao2025snib}, which use reward-baseline adjustment, variational sequence-level weighting, and self-normalized importance sampling, respectively.
As shown in Table~\ref{tab:combined-comparisons}(b), on Qwen3-8B-Base at $n=4$, MinPRO achieves an average \texttt{pass@1} score of $58.3$, outperforming AsymRE ($47.2$), VESPO ($55.1$), and SNIB ($49.1$) by $11.1$, $3.2$, and $9.2$ points, respectively.
These results further demonstrate MinPRO's performance advantage over these baselines in the off-policy setting.

\vspace{1mm}

\textbf{Larger Staleness \ \ }
We also evaluate MinPRO in more off-policy settings with $n=8$ and $n=16$, corresponding to policy lags of $128$--$143$ and $256$--$271$ parameter updates, respectively.
As shown in Table~\ref{tab:combined-comparisons}(c), MinPRO achieves average \texttt{pass@1} scores of $57.3$ and $57.2$, only $1.0$ and $1.1$ points below its score of $58.3$ at $n=4$.
These results demonstrate that MinPRO retains strong performance even under substantially larger policy lags.

\section{Limitations and Conclusion}
\label{sec:discussion}

\paragraph{Limitations}
While our results demonstrate the effectiveness of MinPRO for off-policy LLM training, several limitations remain. First, our training and evaluation are primarily conducted on mathematical reasoning tasks, since there are currently no widely adopted public training datasets for coding or logical reasoning in this off-policy RL setting. Second, although we evaluate MinPRO on dense 8B and 14B models as well as a 30B MoE model, we have not yet tested it on models exceeding 100B parameters. Investigating this regime is non-trivial, as RL post-training for very large LLMs requires substantially greater computational resources, potentially more than an order of magnitude beyond our current setup.

\paragraph{Conclusion}
In this paper, we address the challenge of unstable optimization when training LLMs under off-policy regimes. By revisiting the policy gradient formulation under off-policy conditions, we show that the theoretically rigorous correction term is the prefix importance ratio rather than the commonly used token-level ratio. With this insight, we propose MinPRO, a simple yet effective objective that incorporates prefix-level information through a stable surrogate of cumulative token ratios. MinPRO avoids the numerical instability and length bias of full prefix cumulative products while preserving essential correction signals. Extensive experiments on multiple dense and MoE LLMs across a range of mathematical reasoning benchmarks demonstrate that MinPRO consistently improves both optimization stability and peak performance compared to strong baselines.

\section*{Acknowledgments}

Dr Tao's research is supported by the NTU Research Support Resources Award Number \#023791-00001 and NTU Start Up Grant.

\bibliographystyle{unsrtnat}   
\bibliography{ref}


\appendix

\newpage

\section{Proof}
\label{app:proof}

\subsection{Proof of \Cref{thm:pgt}}
We consider an autoregressive policy $\pi_\theta$ over responses (trajectories) $\bm{o} = (o_1,\dots,o_T)$ with total return
$R(\bm{o}) = \sum_{t=1}^T r_t$.
The standard RL objective is the expected cumulative reward
\begin{equation}
    \mathcal{J}(\theta) \;=\; \mathbb{E}_{\bm{o} \sim \pi_\theta}[R(\bm{o})]. \nonumber
\end{equation}
Using the log–derivative trick (REINFORCE),
\begin{align}
    \nabla_\theta \mathcal{J}
    &= \nabla_\theta \mathbb{E}_{\bm{o}\sim\pi_\theta}[R(\bm{o})]
      = \mathbb{E}_{\bm{o}\sim\pi_\theta}\big[\nabla_\theta \log p_\theta(\bm{o})\,R(\bm{o})\big] \nonumber \\
    &= \mathbb{E}_{\bm{o}\sim\pi_\theta}\left[
        \sum_{t=1}^T \nabla_\theta \log \pi_\theta(o_t \mid \bm{o}_{<t})\,R(\bm{o})
      \right] \nonumber \\
    &= \sum_{t=1}^T 
       \underbrace{\mathbb{E}_{\bm{o}\sim\pi_\theta}
       \left[\nabla_\theta \log \pi_\theta(o_t \mid \bm{o}_{<t})\,R(\bm{o})\right]}_{=:G_t}. \nonumber
\end{align}

Let $\bm{o}_{\leq t} := (o_1,\dots,o_t)$ denote the prefix up to step $t$.
For a fixed $t$, we now rewrite $G_t$ by conditioning on $\bm{o}_{\leq t}$:
\begin{align}
    G_t
    &= \mathbb{E}_{\bm{o}_{\leq t}\sim\pi_\theta}
       \Big[
          \mathbb{E}_{\bm{o}_{t+1:T}\sim\pi_\theta(\cdot\mid \bm{o}_{\leq t})}
          \big[
             \nabla_\theta \log \pi_\theta(o_t \mid \bm{o}_{<t})\,R(\bm{o})
          \,\big|\, \bm{o}_{\leq t}
          \big]
       \Big] 
       && \text{(law of total expectation)} \nonumber \\
    &= \mathbb{E}_{\bm{o}_{\leq t}\sim\pi_\theta}
       \Big[
          \nabla_\theta \log \pi_\theta(o_t \mid \bm{o}_{<t})\;
          \mathbb{E}_{\bm{o}_{t+1:T}\sim\pi_\theta(\cdot\mid \bm{o}_{\leq t})}
          \big[
             R(\bm{o})
          \,\big|\, \bm{o}_{\leq t}
          \big] 
       \Big]. \nonumber
\end{align}

We decompose the return into past and future parts,
\begin{equation}
    R(\bm{o})
    = \underbrace{\sum_{i=1}^{t-1} r_i}_{R_{<t}(\bm{o})}
      + \underbrace{\sum_{i=t}^{T} r_i}_{R_{\ge t}(\bm{o})}. \nonumber
\end{equation}
Given $\bm{o}_{\leq t}$, the past return $R_{<t}$ is deterministic, while the future return $R_{\ge t}$ is random. The state-action value is defined as
\begin{equation}
    Q^\pi(o_t;\bm{o}_{<t})
    := \mathbb{E}_{\bm{o}_{{t+1}:T}\sim\pi_\theta}
       \Big[ R_{\ge t}(\bm{o}) \,\big|\, \bm{o}_{\leq t}\Big]. \nonumber
\end{equation}
Then
\begin{align}
    \mathbb{E}_{\bm{o}_{t+1:T}\sim\pi_\theta(\cdot\mid \bm{o}_{\leq t})}
    \big[R(\bm{o})\mid \bm{o}_{\leq t}\big]
    &= R_{<t}(\bm{o}) + Q^\pi(o_t;\bm{o}_{<t}), \nonumber
\end{align}
and hence 
\begin{align}
    G_t
    &= \mathbb{E}_{\bm{o}_{\leq t}\sim\pi_\theta}
       \Big[
          \nabla_\theta \log \pi_\theta(o_t \mid \bm{o}_{<t})\,
          \big(R_{<t}(\bm{o}) + Q^\pi(o_t;\bm{o}_{<t})\big)
       \Big]. \nonumber
\end{align}

We invoke the following baseline invariance lemma to remove terms that do not depend on the action.
\begin{lemma}[Score-function baseline invariance]
\label{lem:baseline}
Let $b(\bm{o}_{\leq t})$ be any function that depends on $\bm{o}_{<t}$ but not on the current action
$o_t$. Then
\begin{equation}
    \mathbb{E}_{\bm{o}_{\leq t}\sim\pi_\theta}
    \Big[
        \nabla_\theta \log \pi_\theta(o_t \mid \bm{o}_{<t})\,b(\bm{o}_{\leq t})
    \Big] = 0. \nonumber
\end{equation}
\end{lemma}

\begin{proof}
By the tower rule and the definition of conditional expectation,
\begin{align}
    \mathbb{E}_{\bm{o}_{\leq t}\sim\pi_\theta}
    \big[
        \nabla_\theta \log \pi_\theta(o_t \mid \bm{o}_{<t})\,b(\bm{o}_{\leq t})
    \big]
    &= \mathbb{E}_{\bm{o}_{<t}\sim\pi_\theta}
       \Big[
           b(\bm{o}_{\leq t})\;
           \mathbb{E}_{o_t\sim\pi_\theta(\cdot\mid \bm{o}_{<t})}
           \big[
               \nabla_\theta \log \pi_\theta(o_t \mid \bm{o}_{<t})
           \big]
       \Big]. \nonumber
\end{align} 
For any fixed $\bm{o}_{< t}$, the inner expectation is the score-function identity:
\begin{equation}
    \mathbb{E}_{o_t\sim\pi_\theta(\cdot\mid \bm{o}_{<t})}
    \big[
        \nabla_\theta \log \pi_\theta(o_t \mid \bm{o}_{<t})
    \big]
    = \nabla_\theta
      \sum_{o_t} \pi_\theta(o_t \mid \bm{o}_{<t})
    = \nabla_\theta 1
    = 0. \nonumber
\end{equation}
Thus the outer expectation is also zero, which proves the lemma.
\end{proof}

Applying \Cref{lem:baseline} with $b(\bm{o}_{\leq t}) = R_{<t}(\bm{o})$ yields
\begin{align}
    G_t
    &= \mathbb{E}_{\bm{o}_{\leq t}\sim\pi_\theta}
       \Big[
          \nabla_\theta \log \pi_\theta(o_t \mid \bm{o}_{<t})\,Q^\pi(o_t;\bm{o}_{<t})
       \Big]. \nonumber
\end{align}
Recall that in RL the action value can be decomposed as $Q^\pi(o_t;\bm{o}_{<t}) = V^\pi(\bm{o}_{<t}) + A^\pi(o_t;\bm{o}_{<t})$, applying \Cref{lem:baseline} once more with the baseline choice $b(\bm{o}_{\leq t}) = V^\pi(\bm{o}_{<t})$,
we obtain
\begin{align}
    G_t
    &= \mathbb{E}_{\bm{o}_{\leq t}\sim\pi_\theta}
       \Big[
          \nabla_\theta \log \pi_\theta(o_t \mid \bm{o}_{<t})\,A^\pi(o_t;\bm{o}_{<t})
       \Big]. \nonumber
\end{align}
Summing over $t$ gives the on-policy policy gradient
\begin{equation}
    \nabla_\theta \mathcal{J}
    = \sum_{t=1}^T
      \mathbb{E}_{\bm{o}_{\leq t}\sim\pi_\theta}
      \Big[
          \nabla_\theta \log \pi_\theta(o_t \mid \bm{o}_{<t})\,A^\pi(o_t;\bm{o}_{<t})
      \Big]. \nonumber
\end{equation}

\subsection{Proof of \Cref{corollary:main}}

\begin{proof}

By \Cref{thm:pgt}, the policy gradient can be written as
\[
\nabla_{\theta}\mathcal{J}(\theta)
=
\sum_{t=1}^{T}
\mathbb{E}_{o_{1:t}\sim \pi_{\theta}}
\left[
    \nabla_{\theta}\log \pi_{\theta}(o_t \mid o_{<t})
    A^{\pi}(o_t; o_{<t})
\right].
\]
Fix any $t$. Applying the standard change-of-measure identity to the marginal
distribution over the prefix $o_{1:t}$ gives
\[
\begin{aligned}
&\mathbb{E}_{o_{1:t}\sim \pi_{\theta}}
\left[
    \nabla_{\theta}\log \pi_{\theta}(o_t \mid o_{<t})
    A^{\pi}(o_t; o_{<t})
\right]
\\
&\quad =
\sum_{o_{1:t}}
\mathbb{P}_{\theta}(o_{1:t})
\nabla_{\theta}\log \pi_{\theta}(o_t \mid o_{<t})
A^{\pi}(o_t; o_{<t})
\\
&\quad =
\sum_{o_{1:t}}
\mathbb{P}_{\theta_{\mathrm{old}}}(o_{1:t})
\frac{\mathbb{P}_{\theta}(o_{1:t})}
     {\mathbb{P}_{\theta_{\mathrm{old}}}(o_{1:t})}
\nabla_{\theta}\log \pi_{\theta}(o_t \mid o_{<t})
A^{\pi}(o_t; o_{<t})
\\
&\quad =
\mathbb{E}_{o_{1:t}\sim \pi_{\theta_{\mathrm{old}}}}
\left[
    \frac{\mathbb{P}_{\theta}(o_{1:t})}
         {\mathbb{P}_{\theta_{\mathrm{old}}}(o_{1:t})}
    \nabla_{\theta}\log \pi_{\theta}(o_t \mid o_{<t})
    A^{\pi}(o_t; o_{<t})
\right].
\end{aligned}
\]
For an autoregressive policy, the probability of a prefix factorizes as
\[
\mathbb{P}_{\theta}(o_{1:t})
=
\prod_{i=1}^{t}
\pi_{\theta}(o_i \mid o_{<i}),
\qquad
\mathbb{P}_{\theta_{\mathrm{old}}}(o_{1:t})
=
\prod_{i=1}^{t}
\pi_{\theta_{\mathrm{old}}}(o_i \mid o_{<i}).
\]
Therefore,
\[
\frac{\mathbb{P}_{\theta}(o_{1:t})}
     {\mathbb{P}_{\theta_{\mathrm{old}}}(o_{1:t})}
=
\prod_{i=1}^{t}
\frac{
    \pi_{\theta}(o_i \mid o_{<i})
}{
    \pi_{\theta_{\mathrm{old}}}(o_i \mid o_{<i})
}
=
\prod_{i=1}^{t} \rho_i
=: \rho_{1:t}.
\]
Substituting this identity back into the policy-gradient expression yields
\[
\nabla_{\theta}\mathcal{J}(\theta)
=
\sum_{t=1}^{T}
\mathbb{E}_{o_{1:t}\sim \pi_{\theta_{\mathrm{old}}}}
\left[
    \rho_{1:t}
    \nabla_{\theta}\log \pi_{\theta}(o_t \mid o_{<t})
    A^{\pi}(o_t; o_{<t})
\right],
\]
which proves the theorem.
\end{proof}

\section{Implementation Details}
\label{app:implementation details}
The pre-trained LLMs can be downloaded from \url{https://huggingface.co/Qwen}. The training dataset DAPO-Math-17K is available at \url{https://huggingface.co/datasets/BytedTsinghua-SIA/DAPO-Math-17k}, and the evaluation datasets can be downloaded from \url{https://huggingface.co/math-ai}. \Cref{tab:hyperparams} summarizes the key hyperparameters used in the RL algorithms. Experiments on Qwen3-8B-Base and Qwen3-14B-Base are conducted using FSDP training \citep{zhao2023pytorch}, and experiments on Qwen3-30B-A3B-Base employ Megatron \citep{megatron-lm} for distributed training. All experiments use vLLM \citep{kwon2023efficient} as the inference engine.

\begin{table}[h]
\centering
\caption{Key hyperparameters in RL algorithms. ``—'' denotes not used.}
\label{tab:hyperparams}
\small
\resizebox{\textwidth}{!}
{
\begin{tabular}{lccccc}
\toprule
\textbf{Category} & \textbf{GRPO} & \textbf{GSPO} & \textbf{CISPO} & \textbf{M2PO} & \textbf{MinPRO} \\
\midrule
\multicolumn{5}{l}{\textit{Sampling and Validation}} \\
\quad Temperature & $1.0$ & $1.0$ & $1.0$ & $1.0$ & $1.0$ \\
\quad Top-$p$ / Val Top-$p$ & $1.0$ / $0.7$ & $1.0$ / $0.7$ & $1.0$ / $0.7$ & $1.0$ / $0.7$ & $1.0$ / $0.7$ \\
\midrule
\multicolumn{5}{l}{\textit{Clipping}} \\
\quad Clip ratio (low / high) & $0.2$ / $0.28$ & $2e-3$ / $2e-3$ &  $1$ / $4$ & — & $1$ / $4$ \\
\midrule
\multicolumn{5}{l}{\textit{Sequence Limits}} \\
\quad Max prompt / response len & $2048$ / $20480$ & $2048$ / $20480$ & $2048$ / $20480$ & $2048$ / $20480$ & $2048$ / $20480$ \\
\quad Overlong buffer (on/off) & off & off & off & off & off \\
\midrule
\multicolumn{5}{l}{\textit{Batching}} \\
\quad Responses per prompt & $8$ & $8$ & $8$ & $8$ & $8$ \\
\midrule
\multicolumn{5}{l}{\textit{Optimization}} \\
\quad Loss aggregation & \texttt{token-mean} & \texttt{seq-mean} & \texttt{token-mean} & \texttt{token-mean} & \texttt{token-mean} \\
\quad Actor LR & $1e-6$ & $1e-6$ & $1e-6$ & $1e-6$ & $1e-6$ \\
\quad LR warmup steps & $10$ & $10$ & $10$ & $10$ & $10$ \\
\bottomrule
\end{tabular}
}
\end{table}

\begin{figure*}[t]
\centering
\begin{tcolorbox}[
    enhanced,
    colframe=black!22,
    colback=black!1,
    colbacktitle=black!5,
    coltitle=black!78,
    title={\texttt{verl/trainer/ppo/core\_algos.py}\hfill
           },
    fonttitle=\small\bfseries,
    boxrule=0.5pt,
    titlerule=0.4pt,
    arc=3pt,
    left=8pt, right=8pt, top=6pt, bottom=6pt,
    toptitle=5pt, bottomtitle=5pt,
    width=0.98\linewidth
]
\begin{lstlisting}[style=minprocode]
@register_policy_loss("minpro")
def compute_policy_loss_minpro(old_log_prob, log_prob, advantages, response_mask, 
                               loss_agg_mode="token-mean", config=None):
    assert config is not None
    low = config.clip_ratio_low
    high = config.clip_ratio_high
    low = config.clip_ratio if low is None else low
    high = config.clip_ratio if high is None else high

    log_ratio = (log_prob - old_log_prob).clamp(-20, 20)
    ratio = log_ratio.exp()
    valid = response_mask.bool()

    # Exclusive minimum over preceding valid response tokens.
    inclusive_min = ratio.masked_fill(~valid, float("inf")).cummin(-1).values
    prefix_min = torch.cat(
        [torch.ones_like(ratio[..., :1]), inclusive_min[..., :-1]], dim=-1
    )
    prev_valid = valid.long().cumsum(-1) - valid.long()
    prefix_min = torch.where(valid & (prev_valid > 0),
                             prefix_min, torch.ones_like(prefix_min))

    correction = ((ratio * prefix_min).clamp(1 - low, 1 + high)
                  / ratio).detach()
    pg_loss = agg_loss(-advantages * ratio * correction,
                       response_mask, loss_agg_mode)
    ppo_kl = verl_F.masked_mean(-log_ratio, response_mask)
    zero = log_prob.new_zeros(())
    return pg_loss, zero, ppo_kl, zero
\end{lstlisting}
\end{tcolorbox}
\vspace{-1mm}
\captionsetup{
    font=small,
    labelfont=bf,
    justification=raggedright,
    singlelinecheck=false
}
\vspace{-1mm}
\captionof{listing}{\centering MinPRO policy-loss implementation in VeRL.}
\label{lst:minpro_core}
\end{figure*}

\subsection{MinPRO Implementation}
\label{app:minpro_implementation}

Our implementation builds on VeRL\footnote{\url{https://github.com/verl-project/verl}} \citep{sheng2025hybridflow}. The MinPRO loss in Listing~\ref{lst:minpro_core} is added to \path{verl/trainer/ppo/core_algos.py}, registered under \texttt{minpro}, and enabled by setting \path{actor_rollout_ref.actor.policy_loss.loss_mode=minpro}. We adapt the stale-rollout FIFO mechanism and queue checkpoint utilities from the M2PO\footnote{\url{https://github.com/Infini-AI-Lab/M2PO}} \citep{zheng2025prosperity}.

\section{Additional Qwen3-30B-A3B-Base Results}
\label{app:tables}

\begin{figure*}[ht]
\centering
\includegraphics[width=0.32\columnwidth]{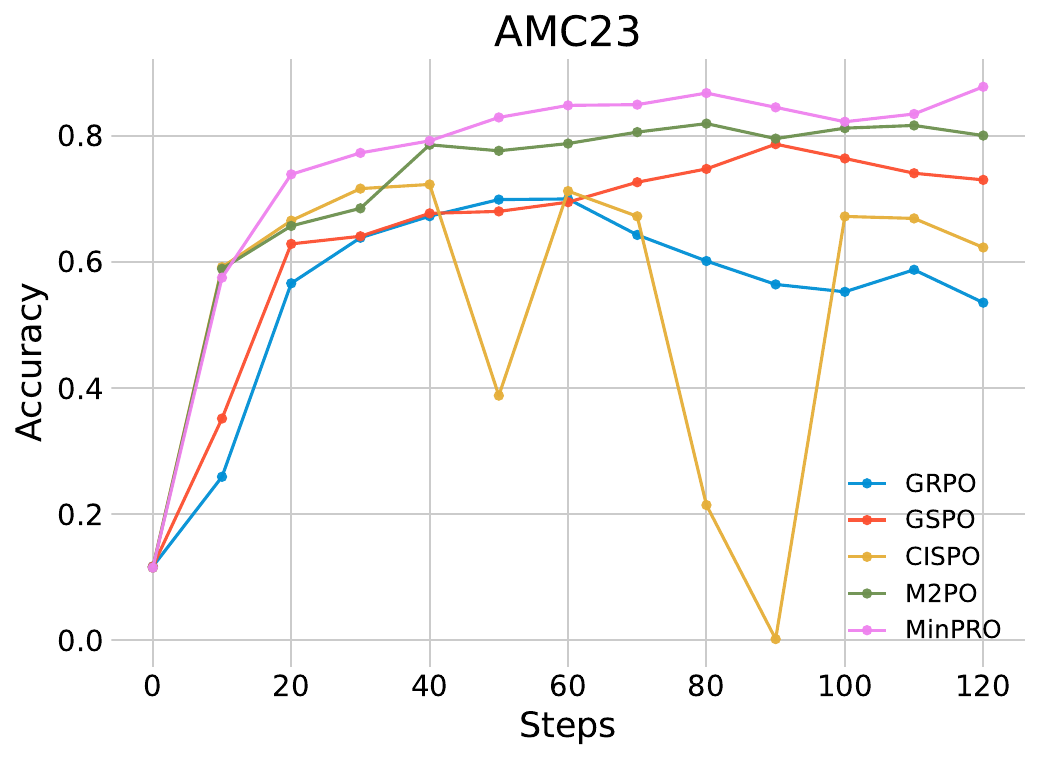}
\includegraphics[width=0.32\columnwidth]{figures/30b-aime24.pdf}
\includegraphics[width=0.32\columnwidth]{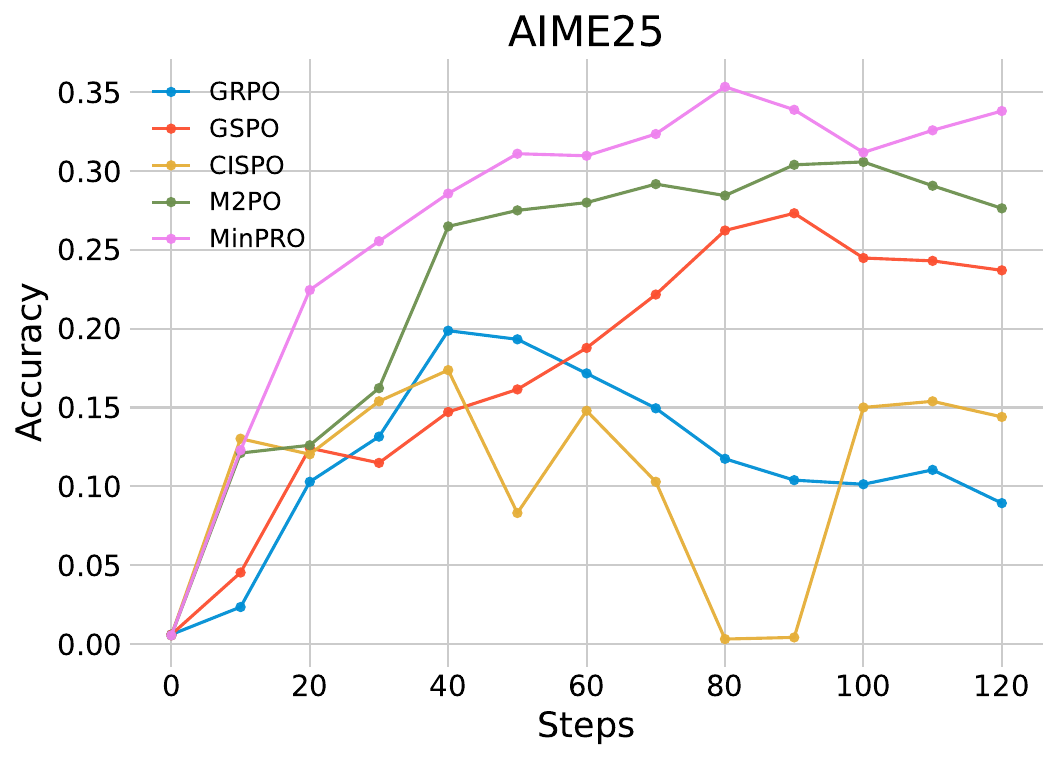}

\hfill

\includegraphics[width=0.32\columnwidth]{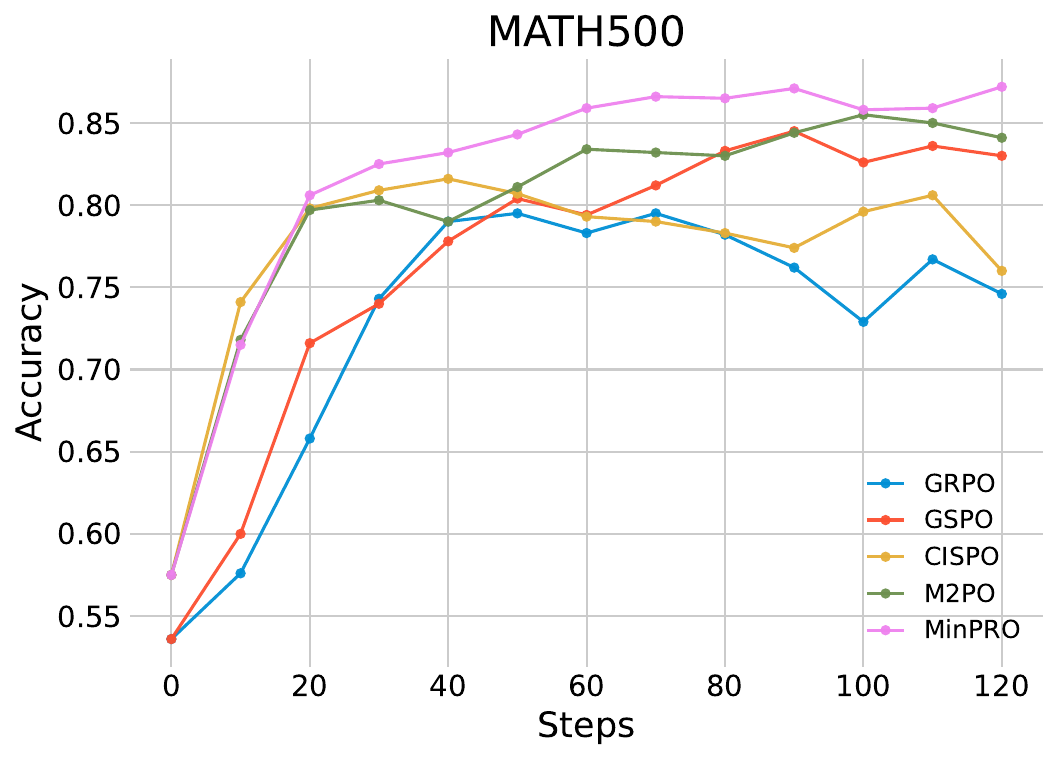}
\includegraphics[width=0.32\columnwidth]{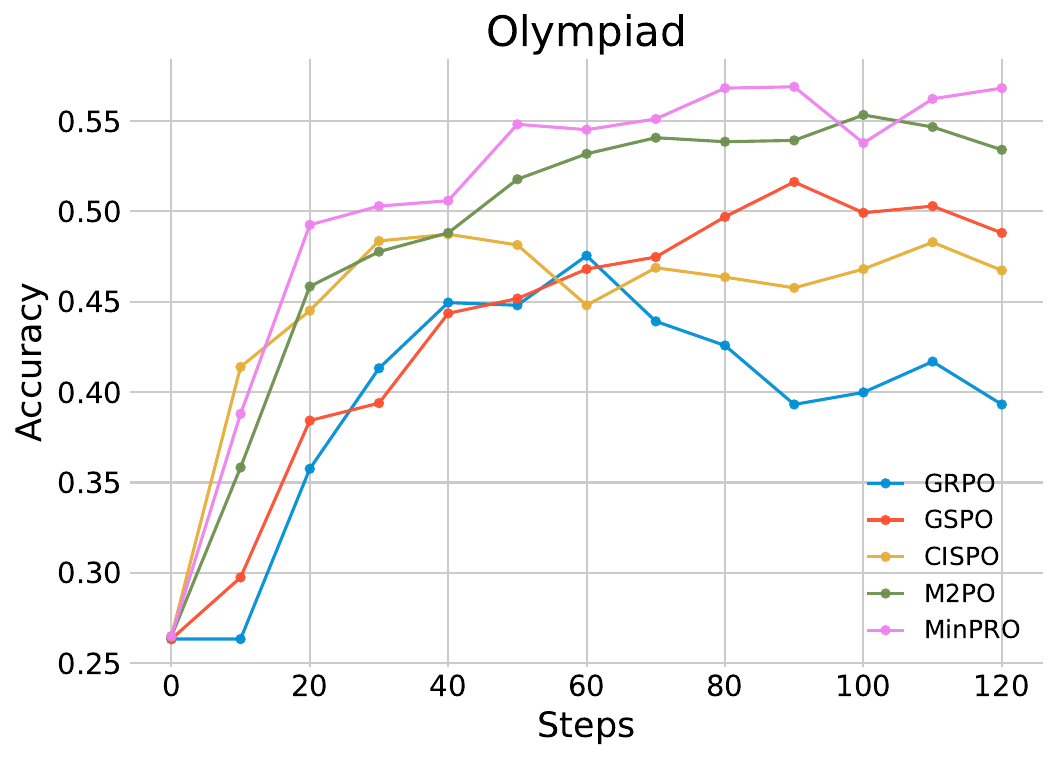}
\includegraphics[width=0.32\columnwidth]{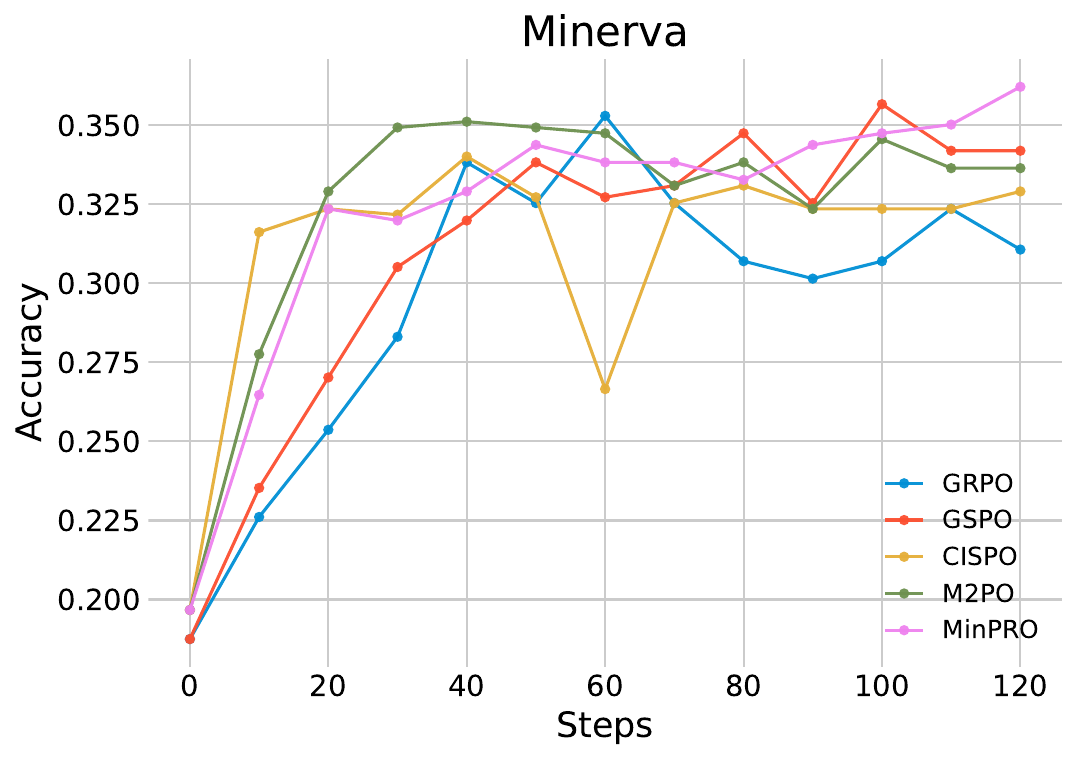}

\hfill

\includegraphics[width=0.32\columnwidth]{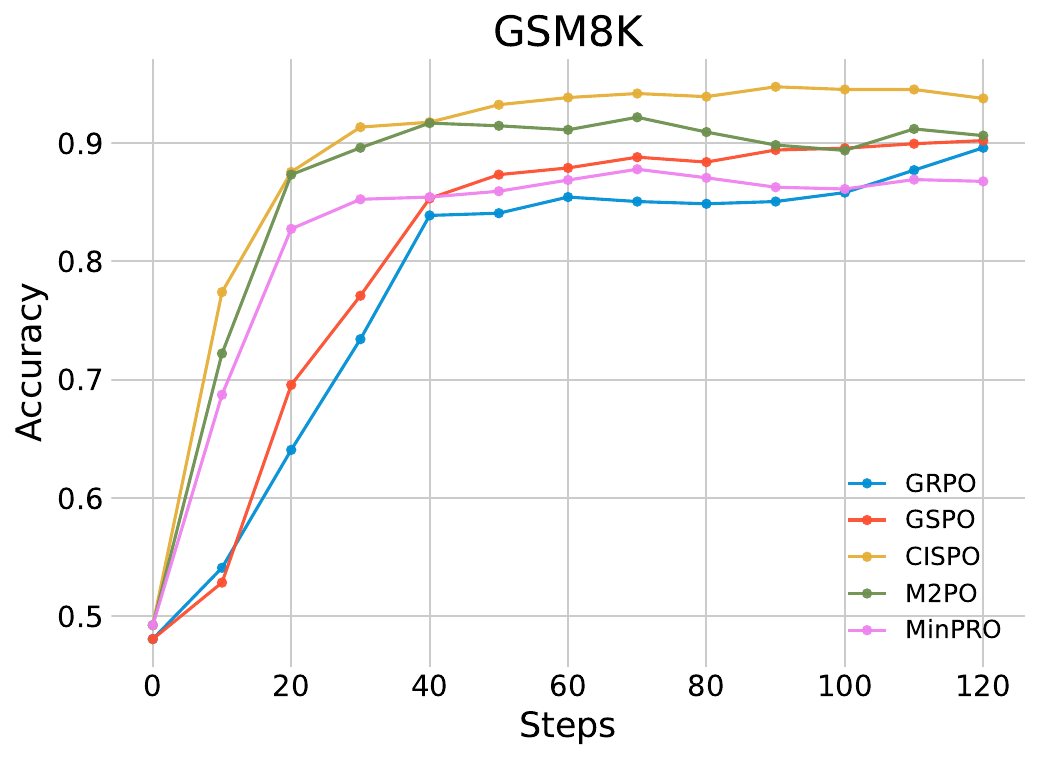}
\caption{Per-dataset \texttt{pass@1} scores with Qwen3-30B-A3B-Base under off-policy training.}
\label{figure:30b pre-data scores}
\end{figure*}

\section{Comparison with AsymRE}
\label{app:asymre}
\begin{table}[ht]
\centering
\caption{\texttt{pass@1} score comparison with AsymRE at $n=2$. The $n=4$ results are in Table~\ref{tab:combined-comparisons}(b).}
\label{tab:asymre_comparison}
\small
\resizebox{\textwidth}{!}
{
\begin{tabular}{llcccccccc}
\toprule
Setting & Method & AMC23 & AIME24 & AIME25 & MATH500 & Olympiad & Minerva & GSM8K & Avg \\
\midrule
\multirow{3}{*}{$n=2$}
& GRPO   & 63.3 & 23.1 & 18.5 & 70.2 & 40.4 & 29.0 & 80.4 & 46.4 \\
& AsymRE & 58.8 & 22.0 & 14.2 & 79.2 & 43.8 & \textbf{34.9} & 85.8 & 48.4 \\
& MinPRO & \textbf{82.5} & \textbf{36.1} & \textbf{30.3} & \textbf{86.5} & \textbf{56.5} & 33.6 & \textbf{90.2} & \textbf{59.4} \\
\bottomrule
\end{tabular}
}
\end{table}

AsymRE \citep{arnal2026asymmetric} is also designed for stabilizing off-policy RL by asymmetrically balancing positive and negative rewards. Its key idea is to reduce the reward baseline, so that positive samples receive larger relative advantages and therefore contribute more strongly to the policy-gradient update. To compare with this direction under our setting, we implement an AsymRE-style baseline by modifying the group-normalized advantage computation. Specifically, for a group of responses with rewards $\{R_j\}$, we compute
\begin{equation}
\hat{A}_i
=
\frac{
R_i -
\left(
\operatorname{mean}\{R_j\}
-
{\color{blue}\alpha \bigl(\operatorname{mean}\{R_j\} - \min\{R_j\}\bigr)}
\right)
}{
\operatorname{std}\{R_j\}
}, \nonumber
\end{equation}
where $\alpha$ controls the amount of baseline reduction. We set $\alpha=0.5$ in our experiments.

We evaluate AsymRE on Qwen3-8B-Base under $n=2$ and $n=4$. Table~\ref{tab:asymre_comparison} reports $n=2$, and Table~\ref{tab:combined-comparisons}(b) in Section~\ref{sec:main-additional-comparisons} reports $n=4$. When $n=2$, AsymRE improves the average pass@1 score from $46.4$ to $48.5$, but its gains are inconsistent across benchmarks. In contrast, MinPRO consistently achieves substantially larger gains, improving the average score to $59.4$.

\clearpage

\section{Position-Conditioned Prefix Factors and Effective Weights}
\label{app:position-diagnostics}

The minimum-prefix factor can attenuate the updates of later tokens after encountering a small ratio in the preceding prefix.
To examine this behavior, we report aggregate diagnostics for Qwen3-8B-Base at $n=4$.
Let $m_t=\underline{\rho}_t$ denote the minimum-prefix factor applied at position $t$.
Table~\ref{tab:prefix-factor-distribution} summarizes its distribution over all tokens and within the first and last deciles of the response.

\begin{table}[htbp]
\centering
\caption{Position-conditioned distribution of the minimum-prefix factor for Qwen3-8B-Base at $n=4$. Threshold columns give the percentage of tokens whose factor exceeds the stated value.}
\label{tab:prefix-factor-distribution}
\small
\setlength{\tabcolsep}{6pt}
\begin{tabular}{@{}lrrrrrr@{}}
\toprule
Token position & Mean & Median & $m>0.5$ & $m>0.2$ & $m>0.1$ & $m>0.05$ \\
\midrule
All tokens & 0.301 & 0.258 & 22.00\% & 58.55\% & 76.05\% & 86.64\% \\
First 10\% & 0.357 & 0.292 & 29.11\% & 62.70\% & 78.39\% & 87.64\% \\
Last 10\% & 0.281 & 0.239 & 19.23\% & 56.05\% & 74.39\% & 85.68\% \\
\bottomrule
\end{tabular}
\end{table}

The mean and median are lower in the last decile than in the first decile, consistent with stronger attenuation at later positions.
A nonzero mean does not exclude a substantial lower tail: overall, $23.95\%$ of factors are at most $0.1$, and $13.36\%$ are at most $0.05$.
In the last decile, $14.32\%$ are at most $0.05$.
Conversely, $76.05\%$ of all factors exceed $0.1$, and $22.00\%$ exceed $0.5$.
These statistics quantify the extent of attenuation, but do not determine whether the downweighted tokens correspond to useful exploratory reasoning.

We also compare the average effective token weights assigned by MinPRO and CISPO in Table~\ref{tab:effective-token-weights}.
The CISPO weight is approximately $1$ across positions, whereas the MinPRO weight decreases from $0.2502$ in the first decile to $0.2036$ in the last decile, an approximately $18.6\%$ relative decrease.
Thus, the reported averages show substantial attenuation relative to CISPO, with additional attenuation at later positions.

\begin{table}[H]
\centering
\caption{Average effective token weights, grouped by relative position in the response, for Qwen3-8B-Base at $n=4$.}
\label{tab:effective-token-weights}
\small
\setlength{\tabcolsep}{16pt}
\begin{tabular}{@{}lrr@{}}
\toprule
Relative token position & MinPRO & CISPO \\
\midrule
0--10\% & 0.2502 & 0.9960 \\
10--20\% & 0.2232 & 1.0007 \\
20--30\% & 0.2185 & 1.0005 \\
30--40\% & 0.2155 & 1.0007 \\
40--50\% & 0.2140 & 1.0006 \\
50--60\% & 0.2129 & 1.0008 \\
60--70\% & 0.2109 & 1.0006 \\
70--80\% & 0.2095 & 1.0002 \\
80--90\% & 0.2087 & 1.0006 \\
90--100\% & 0.2036 & 1.0005 \\
\bottomrule
\end{tabular}
\end{table}

These aggregate diagnostics do not specify the token denominator, exact checkpoint, or aggregation procedure, and do not include uncertainty estimates or a breakdown by training step.
They therefore describe average weighting behavior without establishing that useful exploration is preserved or that attenuation is uniform throughout training.
In particular, a positive sequence-level advantage does not identify the causal contribution of each token, and a nonzero average weight is not a direct measure of successful learning from rare reasoning steps.




\end{document}